\documentclass[11pt,letterpaper]{article}
\pdfoutput=1

\usepackage[utf8]{inputenc}
\usepackage[T1]{fontenc}
\usepackage[english]{babel}
\usepackage{fullpage}
\usepackage{amsmath,amssymb,amsthm}
\usepackage{bm}
\usepackage{mathrsfs}
\usepackage{thmtools}
\usepackage{thm-restate}
\usepackage{mathtools}
\usepackage{comment}
\usepackage{multirow}
\usepackage[textsize=scriptsize]{todonotes}
\usepackage{algpseudocode}
\usepackage{algorithm}
\PassOptionsToPackage{hyphens}{url}
\usepackage{etoolbox}
\usepackage{tikz}
\usepackage{pgfplots}
\usepackage[flushleft]{threeparttable}
\usepackage[hidelinks]{hyperref}
\pgfplotsset{compat=1.5}

\newtoggle{plotexperiments}
\toggletrue{plotexperiments}

\newtoggle{fullversion}
\toggletrue{fullversion}

\declaretheorem{theorem}
\declaretheorem[sibling=theorem]{definition}
\declaretheorem[sibling=theorem]{lemma}
\declaretheorem[sibling=theorem]{corollary}
\declaretheorem[sibling=theorem]{fact}

\newcommand{\newjd}[1]{#1}

\newcommand\numberthis{\addtocounter{equation}{1}\tag{\theequation}}

\DeclareMathOperator*{\argmin}{arg\,min}

\DeclarePairedDelimiter{\norm}{\lVert}{\rVert}

\newcommand{\eps}{\ensuremath{\epsilon}}

\newcommand{\fhat}{\ensuremath{\widehat{f}}}

\newcommand{\OPT}{\ensuremath{\mathrm{OPT}}}

\newcommand{\setI}{\mathcal{I}}

\newcommand{\EE}{\mathbb{E}}

\newcommand{\poly}{\mathrm{poly}}

\newcommand{\Otilde}{\ensuremath{\widetilde{O}}}

\newcommand{\R}{\ensuremath{\mathbb{R}}}

\iftoggle{fullversion}{
}{}

\newcommand{\MSE}{\ensuremath{\mathrm{MSE}}}
\newcommand{\LS}{{\mathrm{LS}}}

\newcommand{\vecfLS}{\vec{f}^{\LS}}
\newcommand{\fLS}{\ensuremath{f_{\LS}}}
\newcommand{\fls}{\ensuremath{f^{\LS}}}
\newcommand{\vecfls}{\ensuremath{\vec{f}^{\LS}}}
\renewcommand{\vec}[1]{\ensuremath{\bm{#1}}}

\algnewcommand{\LineComment}[1]{\State \(\triangleright\) #1}
\makeatletter
\let\OldStatex\Statex
\renewcommand{\Statex}[1][3]{%
\setlength\@tempdima{\algorithmicindent}%
\OldStatex\hskip\dimexpr#1\@tempdima\relax}

\title{Fast Algorithms for Segmented Regression}
\author{Jayadev Acharya\\MIT\\\tt{jayadev@csail.mit.edu}
\and
Ilias Diakonikolas\\University of Southern California\\\tt{ilias.d@ed.ac.uk}
\and
Jerry Li\\MIT\\\tt{jerryzli@csail.mit.edu}
\and
Ludwig Schmidt\\MIT\\\tt{ludwigs@mit.edu}}

\begin{document}
\maketitle

\thispagestyle{empty}

\begin{abstract}
We study the  fixed design segmented regression problem: 
Given noisy samples from a piecewise linear function $f$, 
we want to recover $f$ up to a desired accuracy in mean-squared error.

Previous rigorous approaches for this problem rely on dynamic programming (DP)
and, while sample efficient, have running time quadratic in the sample size. 
As our main contribution, we provide new 
sample near-linear time algorithms for the problem that -- 
while not being minimax optimal -- 
achieve a significantly better sample-time tradeoff 
on large datasets compared to the DP approach.
Our experimental evaluation shows that, compared with the DP approach, 
our algorithms provide a convergence rate that is only off by a factor of $2$ to $4$, 
while achieving speedups of three orders of magnitude.

\end{abstract}
\newpage

\thispagestyle{empty}

\section{Introduction}
\label{sec:intro}

We study the {\em regression} problem  -- a fundamental inference task
with numerous applications that has received tremendous attention 
in machine learning and statistics during the past fifty years 
(see, e.g., ~\cite{MT77} for a classical textbook). 
Roughly speaking, in a (fixed design) regression problem,
we are given a set of $n$ observations $(\mathbf{x}_i, y_i)$, where the $y_i$'s are the dependent
variables and the $\mathbf{x}_i$'s are the independent variables, and our
goal is to  model the relationship between them. The typical assumptions are that 
(i) there exists a simple function $f$ that (approximately) models the underlying relation, 
and (ii) the dependent observations are corrupted by random noise. More specifically, we assume that 
there exists a family of functions $\cal F$ such that for some $f \in {\cal F}$ the following holds: 
$y_i = f(\mathbf{x}_i)+\eps_i,$ where the $\eps_i$'s are i.i.d.\ random variables drawn from a ``tame'' distribution such as a Gaussian (later, we also consider model misspecification).

Throughout this paper, we consider the classical notion of Mean Squared Error (MSE) to measure 
the performance (risk) of an estimator. As expected, the minimax risk depends on the family $\cal F$ that $f$
comes from. The natural case that $f$ is  linear is fully understood: It is well-known that the least-squares
estimator is statistically efficient and runs in sample-linear time. The more general case that $f$ is {\em non-linear}, 
but satisfies some well-defined structural constraint has been extensively studied 
in a variety of contexts (see, e.g., ~\cite{GF73, feder1975, friedman1991, BP98, YP13, KRS15, ASW13, meyer2008, chatterjee2015}).
In contrast to the linear case, this more general setting is not well-understood from an information-theoretic
and/or computational aspect.

\iftoggle{fullversion}{
}{
\footnotetext{Authors ordered alphabetically.}
}

In this paper, we focus on the case that the function $f$ is promised 
to be {\em piecewise linear} with a given number $k$ 
of {\em unknown} pieces (segments). This is known as fixed design {\em segmented} regression, 
and has received considerable attention in the statistics community
~\cite{GF73, feder1975, BP98, YP13}. 
The special case of piecewise polynomial functions (splines) has been extensively 
used  in the context of inference, including density estimation and regression,
see, e.g.,~\cite{WegW83, friedman1991, Stone94, Stone97, meyer2008}.

Information-theoretic aspects of the segmented regression problem are well-understood:
Roughly speaking, the minimax risk \newjd{is inversely proportional} to the number of samples. 
In contrast, the computational complexity of the problem is poorly understood:
Prior to our work, known algorithms for this problem with provable guarantees were quite slow.
Our main contribution is a set of new {\em provably fast} algorithms that outperform 
previous approaches both in theory and in practice. 
Our algorithms run in time that is nearly-linear in the number of data points $n$ and the number of intervals $k$.
Their computational efficiency is established both theoretically and experimentally.
We also emphasize that our algorithms are robust to model misspecification, 
i.e., \newjd{they} perform well even if the function $f$ is only {\em approximately} piecewise linear.

Note that if the segments of $f$ were known a priori, 
the segmented regression problem could be immediately reduced to $k$ independent linear regression problems. 
Roughly speaking, in the general case (where the location of the segment boundaries is unknown) 
one needs to ``discover'' the right segments using information provided by the samples.
To address this algorithmic problem,  previous works~\cite{BP98, YP13} relied on
dynamic programming that, while being statistically efficient, is computationally
quite slow: its running time scales at least {\em quadratically} 
with the size $n$ of the data, hence it is rather impractical for large datasets.

Our main motivation comes from the availability of large datasets that has made computational efficiency the main bottleneck in many cases.
In the words of~\cite{jordan2013}: 
``As data grows, it may be beneficial to consider faster inferential algorithms, because the increasing statistical strength of the data can compensate for the poor algorithmic quality.''
Hence, it is sometimes advantageous to sacrifice statistical efficiency in order to achieve faster running times because we can then achieve the desired error guarantee faster (provided more samples).
In our context, instead of using a slow dynamic program, we employ
a subtle iterative greedy approach that runs in sample-linear time. 

Our iterative greedy approach builds on the work of~\cite{ADHLS15,ADLS15}, but the details of our algorithms here and their analysis are substantially different. 
In particular, as we explain in the body of the paper, 
the natural adaptation of their analysis to our setting fails to provide 
any meaningful statistical guarantees.
To obtain our results, we introduce novel algorithmic ideas and carefully combine them with additional probabilistic arguments.

\section{Preliminaries}
In this paper, we study the problem of fixed design segmented regression.
We are given samples $\vec{x}_i \in \R^d$ for $i \in [n] \, (\, = \{1,  \ldots, n\})$, and we consider the following classical regression model:
\begin{equation}
\label{eq:model}
y_i = f(\vec{x}_i) + \epsilon_i \; .
\end{equation}
Here, the $\epsilon_i$ are i.i.d.\ sub-Gaussian noise variables with variance proxy $\sigma^2$, mean $\EE [\epsilon_i] = 0$, and variance $s^2 = \EE[\epsilon_i^2]$ for all $i$.\footnote{We observe that $s^2$ is guaranteed to be finite since the $\epsilon_i$ are sub-Gaussian. The variance $s^2$ is in general not equal to the variance proxy $\sigma^2$, however, it is well-known that $s^2 \leq \sigma^2$.}
We will let $\epsilon = (\epsilon_1, \ldots, \epsilon_n)$ denote the vector of noise variables.
We also assume that $f: \R^d \to \R$ is a $k$-piecewise linear function. 
Formally, this means:
\begin{definition}
The function $f: \R^d \to \R$ is a \emph{$k$-piecewise linear} function if there exists a partition of the real line into $k$ disjoint intervals $I_1, \ldots, I_k$, $k$ corresponding parameters $\vec{\theta}_1, \ldots, \vec{\theta}_k \in \R^d$, and a fixed, known $j$ such that for all $\vec{x} = (x_1, \ldots, x_d) \in \R^d$ we have that $f(\vec{x}) = \langle \vec{\theta}_i, \vec{x} \rangle$ if $x_j \in I_i$.
Let $\mathcal{L}_{k,j}$ denote the space of $k$-piecewise linear functions with partition defined on coordinate $j$.

Moreover, we say $f$ is \emph{flat} on an interval $I \subseteq \R$ if $I \subseteq I_i$ for some $i = 1, \ldots, k$, otherwise, we say that $f$ has a \emph{jump} on the interval $I$. \end{definition}
Later in the paper (see Section \ref{sec:agnostic}), we also discuss the setting where the ground truth $f$ is not piecewise linear itself but only well-approximated by a $k$-piecewise linear function.
For simplicity of exposition, we assume that the partition coordinate $j$ is $1$ in the rest of the paper.

Following this generative model, a regression algorithm receives the $n$ pairs $(\vec{x}_i, y_i)$ as input.
The goal of the algorithm is then to produce an estimate $\fhat$ that is close to the true, unknown $f$ with high probability over the noise terms $\epsilon_i$ and any randomness in the algorithm.
We measure the distance between our estimate $\fhat$ and the unknown function $f$ with the classical mean-squared error:
\[\MSE (\fhat) = \frac{1}{n} \sum_{i = 1}^n (f(\vec{x_i}) - \fhat(\vec{x_i}))^2 \; .\]

Throughout this paper, we let $\vec{X} \in \R^{n \times d}$ be the \emph{data matrix}, i.e., the matrix whose $j$-th row is $\vec{x}_j^T$ for every $j$, and we let $r$ denote the rank of $\vec{X}$. 
We also assume that no $\vec{x}_i$ is the all-zeros vector, since such points are trivially fit by any linear function.

The following notation will also be useful.
For any function $f: \R^d \to \R$, we let $\vec{f} \in \R^{n}$ denote the vector with components $\vec{f}_i = f(\vec{x}_i)$ for \newjd{$i \in [n]$}.
 For any interval $I$, we let $\vec{X}^I$ denote the data matrix consisting of all data points $\vec{x}_i$ for $i \in I$, and for any vector $\vec{v} \in \R^n$, we let $\vec{v}^I \in \R^{|I|}$ be the vector of $\vec{v}_i$ for $i \in I$.

We remark that this model also contains the problem of (fixed design) piecewise polynomial regression as an important subcase. 
Indeed, imagine we are given $x_1, \ldots, x_n \in \R$ and $y_1, \ldots, y_n \in \R$ so that there is some $k$-piece degree-$d$ polynomial $p$ so that for all $i = 1, \ldots, n$ we have $y_i = p(x_i) + \epsilon_i$, and our goal is to recover $p$ in the same mean-squared sense as above.
We can write this problem as a $k$-piecewise linear fit in $\R^{d + 1}$, by associating the data vector $\vec{v}_i = (1, x_i, x_i^2, \ldots, x_i^{d})$ to each $x_i$.
If the piecewise polynomial $p$ has breakpoints at $z_1, \ldots z_{k - 1}$, the associated partition for the $k$-piecewise linear function is $\setI = \{ (-\infty, z_1), [z_1, z_2), \ldots, (z_{k - 2}, z_{k - 1}], (z_{k-1}, \infty) \}$.
If the piecewise polynomial is of the form $p(x) = \sum_{\ell = 0}^d a^I_\ell x^\ell$ for any interval $I \in \setI$, 
then for any vector $\vec{v} = (v_1, v_2, \ldots, v_{d+1}) \in \R^{d + 1}$, the ground truth $k$-piecewise linear function is simply the linear function $f(\vec{v}) = \sum_{\ell = 1}^{d+1} a^I_{\ell-1} v_\ell$ for $v_2 \in I$.
Moreover, the data matrix associated with the data points $\vec{v}$ is a Vandermonde matrix.
For $n \geq d + 1$ it is well-known that this associated Vandermonde matrix has rank exactly $d + 1$ (assuming $x_i \neq x_j$ for any $i \neq j$).

\subsection{Our Contributions}
Our main contributions are new, fast algorithms for the aforementioned segmented regression problem.
We now informally state our main results and refer to later sections for more precise theorems.
\begin{theorem}[informal statement of Theorems \ref{thm:time-greedy} and \ref{thm:greedy}]
\label{thm:inf-greedy}
There is an algorithm \textsc{GreedyMerge}, which, given $\vec{X}$ (of rank $r$), $\vec{y}$, a target number of pieces $k$, and the variance of the noise $s^2$, runs in time $O(n d^2 \log n)$ and outputs an $O(k)$-piecewise linear function $\fhat$ so that with probability at least $0.99$, we have
\[
  \MSE (\fhat) \; \leq \; O\left(\sigma^2 \frac{kr}{n} + \sigma \sqrt{\frac{k}{n}} \log n \right) \; .
\]
\end{theorem}
That is, our algorithm runs in time which is \emph{nearly linear} in $n$ and still achieves a reasonable rate of convergence.
While this rate is asymptotically slower than the rate of the dynamic programming estimator, our algorithm is significantly faster than the DP so that in order to achieve a comparable MSE, our algorithm takes less total time given access to a sufficient number of samples.
For more details on this comparision and an experimental evaluation, see Sections \ref{sec:prior} and \ref{sec:experiments}.

At a high level, our algorithm proceeds as follows: it first forms a fine partition of $[n]$ and then iteratively merges pieces of this partitions until only $O(k)$ pieces are left.
In each iteration, the algorithm reduces the number of pieces in the following manner: we group consecutive intervals into pairs which we call ``candidates''.
For each candidate interval, the algorithm computes an error term that is the error of a least squares fit combined with a regularizer depending on the variance of the noise $s^2$.
The algorithm then finds the $O(k)$ candidates with the largest errors.
We do not merge these candidates, but do merge all other candidate intervals.
We repeat this process until only $O(k)$ pieces are left.

A drawback of this algorithm is that we need to know the variance of the noise $s^2$, or at least have a good estimate of it.
In practice, we might be able to obtain such an estimate, but ideally our algorithm would work without knowing $s^2$.
By extending our greedy algorithm, we obtain the following result:
\begin{theorem}[informal statement of Theorems \ref{thm:time-bucket} and \ref{thm:bucket}]
There is an algorithm \textsc{BucketGreedyMerge}, which, given $\vec{X}$ (of rank $r$), $\vec{y}$, and a target number of pieces $k$, runs in time $O(n d^2 \log n)$ and outputs an $O(k \log n)$-piecewise linear function $\fhat$ so that with probability at least $0.99$, we have
\[
  \MSE (\fhat) \; \leq \; O\left( \sigma^2 \frac{kr \log n}{n} + \sigma \sqrt{\frac{k}{n}} \log n \right) \; .
\]
\end{theorem}

At a high level, there are two fundamental changes to the algorithm: first, instead of merging with respect to the sum squared error of the least squares fit regularized by a term depending on $s^2$, we instead merge with respect to the average error the least squares fit incurs within the current interval.
The second change is more substantial: instead of finding the top $O(k)$ candidates with largest error and merging the rest, we now split the candidates into $\log n$ buckets based on the lengths of the candidate intervals.
In this bucketing scheme, bucket $\alpha$ contains all candidates with length between $2^\alpha$ and $2^{\alpha + 1}$, for $\alpha = 0, \ldots, \log n - 1$.
Then we find the $k + 1$ candidates with largest error within each bucket and merge the remaining candidate intervals.
We continue this process until we are left with $O(k \log n)$ buckets.
Intuitively, this bucketing allows us to control the variance of the noise without knowing $s^2$ because all candidate intervals have roughly the same length.

A potential disadvantage of our algorithms above is that they produce $O(k)$ or $O(k \log n)$ pieces, respectively.
In order to address this issue, we also provide a postprocessing algorithm that converts the output of any of our algorithms and decreases the number of pieces down to $2k  + 1$ while preserving the statistical guarantees above.
The guarantee of this postprocessing algorithm is as follows.
\begin{theorem}[informal statement of Theorems \ref{thm:time-postprocessing} and \ref{thm:postprocessing}]
There is an algorithm \textsc{Postprocessing} that takes as input the output of either $\textsc{GreedyMerge}$ or $\textsc{BucketGreedyMerge}$ together with a target number of pieces $k$, runs in time $O\left( k^3 d^2 \log n\right)$, and outputs a $(2k + 1)$-piecewise linear function $\fhat^p$ so that with probability at least $0.99$, we have
\[
  \MSE (\fhat^p) \; \leq \; O\left(\sigma^2 \frac{kr}{n} + \sigma \sqrt{\frac{k}{n}} \log n \right) \; .
\]
\end{theorem}
Qualitatively, an important aspect this algorithm is that its running time depends only logarithmically on $n$.
In practice, we expect $k$ to be much smaller than $n$, and hence the running time of this postprocessing step will usually be dominated by the running time of the piecewise linear fitting algorithm run before it.

\subsection{Comparison to prior work}
\label{sec:prior}

\paragraph{Dynamic programming.}
Previous work on segmented regression with statistical guarantees ~\cite{BP98, YP13} relies heavily on dynamic programming-style algorithms to find the $k$-piecewise linear least-squares estimator.
Somewhat surprisingly, we are not aware of any work which explicitly gives the best possible running time and statistical guarantee for this algorithm.
For completeness,  we prove the following result \newjd{(Theorem~\ref{thm:inf-DP}), which we believe to be folklore. The techniques used in its proof will also be useful for us later.}
\begin{theorem}[informal statement of Theorems \ref{thm:time-DP} and \ref{thm:DP}]
\label{thm:inf-DP}
The exact dynamic program runs in time $O(n^2 (d^2 + k ))$ and outputs an $k$-piecewise linear estimator $\fhat$ so that with probability at least $0.99$ we have 
\[
  \MSE(\fhat) \; \leq \; O\left( \sigma^2 \frac{kr}{n} \right) \; .
\]
\end{theorem}
We now compare our guarantees with those of the DP.
The main advantage of our approaches is computational efficiency -- our algorithm runs in linear time, while the running time of the DP has a quadratic dependence on $n$.
While our statistical rate of convergence is slower, we are able to achieve the same MSE as the DP in asymptotically less time (and also in practice) if enough samples are available.

For instance, suppose we wish to obtain a MSE of $\eta$ with a $k$-piecewise linear function, and suppose for simplicity that $d = O(1)$ so that $r = O(1)$ as well.
Then Theorem \ref{thm:inf-DP} tells us that the DP requires $n = k / \eta$ samples and runs in time $O(k^3 / \eta^2)$.
On the other hand, Theorem \ref{thm:inf-greedy} tells us that \textsc{GreedyMerging} requires $n = \Otilde (k / \eta^2)$ samples (ignoring log factors) and thus runs in time $\Otilde(k / \eta^2)$.
For non-trivial values of $k$, this is a significant improvement in time complexity.

This gap also manifests itself strongly in our experiments (see Section \ref{sec:experiments}).
When given the same number of samples, our MSE is a factor of 2-4 times worse than that achieved by the DP, but our algorithm is about $1,000$ times faster already for  $10^4$ data points.
When more samples are available for our algorithm, it achieves the same MSE as the DP about $100$ times faster.

\paragraph{Histogram Approximation}
Our results build upon the techniques of \cite{ADHLS15}, who consider the problem of \emph{histogram approximation}.
In this setting, one is given a function $f: [n] \to \R$, and the goal is to find the best $k$-flat approximation to $f$ in sum-squared error.
This problem bears nontrivial resemblance to the problem of piecewise constant regression, and indeed, the results in this paper establish a connection between the two problems.
The histogram approximation problem has received significant attention in the database community (e.g. \cite{Jagadish98, GKS06, ADHLS15}), and it is natural to ask whether these results also imply algorithms for segmented regression.
Indeed, it is possible to convert the algorithms of \cite{GKS06} and \cite{ADHLS15} to the segmented regression setting, but the corresponding guarantees are too weak to obtain good statistical results.
At a high level, the algorithms in these papers can be adapted to output an $O(k)$-piecewise linear function $\fhat$ so that $\sum_{i = 1}^n \newjd{(y_i - \fhat(\vec{x}_i))^2} \leq C \cdot\sum_{i = 1}^n (y_i - \fls(\vec{x}_i))^2$, where $C > 1$ is a fixed constant and $\fls$ is the best $k$-piecewise linear least-squares fit to the data.
However, this guarantee by itself does not give meaningful results for segmented regression.

As a toy example, consider the $k = 1$ case.
Let $\vec{x}_1, \ldots, \vec{x}_n \in \R^n$ be arbitrary, let $f(\vec{x}) = 0$ for all $\vec{x}$, and $\epsilon_i \sim \mathcal{N} (0, 1)$, for $i = 1, \ldots, n$.
Then it is not too hard to see that the least squares fit has squared error $\sum_{i=1}^n (y_i - \fls(\vec{x}_i))^2 = \Theta (n)$ with high probability.
Hence the function $g$ which is constantly $C / 2$ for all $\vec{x}$ achieves the approximation guarantee
\[
  \sum_{i=1}^n (y_i - g(\vec{x}_i))^2 \; \leq \; C \sum_{i = 1}^n (y_i - \fls (\vec{x}_i))^2
\]
described above with non-negligible probability.
However, this function clearly does not converge to $f$ as $n \to \infty$, and indeed its MSE is always $\Theta(1)$.

To get around this difficulty, we must extend the algorithms presented in histogram approximation literature in order to adapt them to the segmented regression setting.
In the process, we introduce new algorithmic ideas and use more sophisticated proof techniques to obtain a meaningful rate of convergence for our algorithms.

\subsection{Agnostic guarantees}
\label{sec:agnosticintro}
We also consider the agnostic model, also known as learning with \emph{model misspecification}.
So far, we assumed that the ground truth is exactly a piecewise linear function.
In reality, such a notion is probably only an approximation.
While the ground truth may be close to a piecewise linear function, generally we do not believe that it exactly follows a piecewise linear function.
In this case, our goal should be to recover a piecewise linear function that is competitive with the best piecewise linear approximation to the ground truth.

Formally, we consider the following problem.
We assume the same generative model as \newjd{in~\eqref{eq:model}}, however, we no longer assume that $f$ is a piecewise linear function.
Instead, the function $f$ can now be arbitrary.
We define
\[\OPT_k = \min_{g \in \mathcal{L}_k} \MSE(g) \]
to be the error of the best fit $k$-piecewise linear function to $f$, and we let $f^\ast$ be any $k$-piecewise linear function that achieves this minimum.
Then the goal of our algorithm is to achieve an MSE as close to $\OPT_k$ as possible.
We remark that the qualitatively interesting regime is when $\OPT_k$ is small and comparable to the statistical error, and we encourage readers to think of $\OPT_k$ as being in that regime.

For clarity of exposition, we first prove non-agnostic guarantees. 
In Section \ref{sec:agnostic}, we then show how to modify these proofs to obtain agnostic guarantees with roughly the same rate of convergence.

\subsection{Mathematical Preliminaries}
In this section, we state some preliminaries that our analysis builds on.

\paragraph{Tail bounds.}
We require the following tail bound on sub-exponential random variables.
Recall that for any sub-exponential random variable $X$, the \emph{sub-exponential norm} of $X$, denoted $\| X \|_{\psi_1}$, is defined to be the smallest parameter $K$ so that $(\EE [|X|^p])^{1/p} \leq K p$ for all $p \geq 1$ (see \cite{Vershynin}).
Moreover, if $Y$ is a sub-Gaussian random variable with variance $\sigma^2$, then $Y^2 - \sigma^2$ is a centered sub-exponential random with $\| Y^2 - \sigma^2 \|_{\psi_1} = O(\sigma^2)$.
\begin{fact}[Bernstein-type inequality, c.f. Proposition 5.16 in \cite{Vershynin}]
\label{thm:bernstein}
Let $X_1, \ldots, X_N$ be centered sub-exponential random variables, and let $K = \max_i \| X_i \|_{\psi_1}$.
Then for all $t > 0,$ we have
\[
  \Pr \left[ \left| \frac{1}{\sqrt{N}} \sum_{i = 1}^N X_i \right| \; \geq \; t \right] \leq 2 \exp \left( -c \min \left( \frac{t^2}{K^2}, \frac{t \sqrt{N}}{K} \right) \right) \; .
\]
\end{fact}

This yields the following straightforward corollary:
\begin{corollary}
\label{cor:unif-err-bound}
Fix $\delta > 0$ and let $\epsilon_1, \ldots, \epsilon_n$ be as in~\eqref{eq:model}.
Recall that $s = \EE [\epsilon_i^2]$.
Then, with probability $1 - \delta$, we have that simultaneously for all intervals $I \subseteq [n]$ the following inequality holds:
\begin{equation}
\label{eq:unif-err-bound}
 \left| \sum_{i \in I} \epsilon_i^2 - s^2 |I| \right| \; \leq \; O(\sigma^2 \cdot \log (n / \delta)) \sqrt{|I|} \; .
 \end{equation}
\end{corollary}
\begin{proof}
Let $\delta' = O(\delta / n^2)$. For any interval $I \subseteq [n]$, by Fact \ref{thm:bernstein} applied to the random variables $\epsilon_i^2 - s^2$ for $i \in I$, we have that \eqref{eq:unif-err-bound} holds with probability $1 - \delta'$.
Thus, by a union bound over all $O(n^2)$ subintervals of $[n]$, we conclude that Equation \ref{eq:unif-err-bound} holds with probability $1 - \delta$.
\end{proof}

\paragraph{Linear regression.}
Our analysis builds on the classical results for fixed design linear regression.
In linear regression, the generative model is exactly of the form described in \eqref{eq:model}, except that $f$ is restricted to be a $1$-piecewise linear function (as opposed to a $k$-piecewise linear function), i.e., $f(\vec{x}) = \langle \vec{\theta}^\ast, \vec{x} \rangle$ for some unknown $\vec{\theta}^\ast$.

The problem of linear regression is very well understood, and the asymptotically best estimator is known to be the \emph{least-squares} estimator.
\begin{definition}
Given $\vec{x}_1, \ldots, \vec{x}_n$ and $y_1, \ldots, y_n$, the \emph{least squares estimator} $\fls$ is defined to be the linear function which minimizes $\sum_{i = 1}^n (y_i - f(\vec{x}_i))^2$ over all linear functions $f$.
For any interval $I$, we let $\fls_I$ denote the least squares estimator for the data points restricted to $I$, i.e. for the data pairs $\{(\vec{x}_i, y_i) \}_{i \in I}$.
We also let \textsc{LeastSquares} (\vec{X}, \vec{y}, I) denote an algorithm which solves linear least squares for these data points, i.e., which outputs the coefficients of the linear least squares fit for these points.
\end{definition}
Following our previous definitions, we let $\vecfls \in \R^n$ denote the vector whose $i$-th coordinate is $\fls(\vec{x}_i)$, and similarly for any $I \subseteq [n]$ we let $\vecfls_I \in \R^{|I|}$ denote the vector whose $i$-th coordinate for $i \in I$ is $\fls_I (\vec{x}_i)$.

The following prediction error rate is known for the least-squares estimator:
\begin{fact}
\label{thm:LSrate}
Let $\vec{x}_1, \ldots, \vec{x}_n$ and $y_1, \ldots, y_n$ be generated as in~\eqref{eq:model}, where $f$ is a linear function. Let $\fls (x)$ be the least squares estimator. Then,
\[\EE \left[ \MSE (\fls) \right] = O \left( \sigma^2 \frac{r}{n} \right) \; .\]
Moreover, with probability $1 - \delta$, we have
\[\MSE (\fls) = O\left( \sigma^2 \, \frac{r + \log 1 / \delta}{n} \right) \; . \]
\end{fact}

Fact \ref{thm:LSrate} can be proved with the following lemma, which we also use in our analysis.
The lemma bounds the correlation of a random vector with any fixed r-dimensional subspace:
\begin{lemma}[c.f. proof of Theorem 2.2 in \cite{Rig15}]
\label{lem:incoherence}
Fix $\delta > 0$.
Let $\epsilon_1, \ldots, \epsilon_n$ be i.i.d.\ sub-Gaussian random variables with variance proxy $\sigma^2$.
Let $\vec{\epsilon} = (\epsilon_1, \ldots, \epsilon_n)$, and let $S$ be a fixed $r$-dimensional affine subspace of $\R^n$.
Then, with probability $1 - \delta$, we have
\[\sup_{\vec{v} \in S \setminus \{0\}} \frac{| \vec{v}^T \vec{\epsilon} |}{\| \vec{v} \|} \; \leq \; O\left(\sigma \sqrt{r + \log (1 / \delta})\right) \; . \]
\end{lemma}

This lemma also yields the two following consequences.
The first corollary bounds the correlation between sub-Gaussian random noise and any linear function on any interval:
\begin{corollary}
\label{cor:int-incoherence}
Fix $\delta > 0$ and $\vec{v} \in \R^n$.
Let $\vec{\epsilon} = (\epsilon_1, \ldots, \epsilon_n)$ be as in Lemma \ref{lem:incoherence}.
Then with probability at least $1 - \delta$, we have that for all intervals $I$, and for all non-zero linear functions $f: \R^d \to \R$, 
\[
  \frac{|\langle \vec{\epsilon}^I, \vec{f}_I + \vec{v}_I \rangle|}{ \| \vec{f}_I + \vec{v}_I \|} \; \leq \; O(\sigma \sqrt{r + \log (n / \delta)}) \; .
\]
\end{corollary}
\begin{proof}
Fix any interval $I \subseteq [n]$.
Observe that any linear function on $I$ can only take values in the range $\{ \vec{X}^I \vec{\theta} : \vec{\theta} \in \R^d\}$, and hence the range of functions of the form $f(\vec{x}_i) + \vec{v}$ is at most an $r$-dimensional affine subspace.
Thus, by Lemma \ref{lem:incoherence}, we know that for any linear function $f$,
\[
  \frac{|\langle \vec{\epsilon}_I, \vec{f}_I + \vec{v}_I \rangle|}{ \| \vec{f}_I + \vec{v}_I \|} \; \leq \; O(\sigma \sqrt{r + \log (n / \delta)}) \; .
\]
with probability $1 - O(\delta / n^2)$.
By union bounding over all $O(n^2)$ intervals, we achieve the desired result.
\end{proof}

\begin{corollary}
\label{cor:sim-incoherence}
Fix $\delta > 0$ and $\vec{v} \in \R^n$.
Let $\vec{\epsilon} = (\epsilon_1, \ldots, \epsilon_n)$ be as in Lemma \ref{lem:incoherence}.
Then with probability at least $1 - \delta$, we have
\[
  \sup_{f \in \mathcal{L}_k} \frac{\left| \langle \vec{\epsilon}, \vec{f}_I+ \vec{v}_I \rangle \right| }{\| \vec{f}_I + \vec{v}_I \|}  \; \leq \; O \left( \sigma \sqrt{k r + k \log \frac{n}{\delta}} \right) \; .
\]
\end{corollary}
\begin{proof}
Fix any partition of $[n]$ into $k$ intervals $\setI$, and let $S_\setI$ be the set of $k$-piecewise linear functions which are flat on each $I \in \setI$.
It is not hard to see that $S_\setI$ is a $kr$-dimensional linear subspace, and hence the set of all $k$-piecewise linear functions which are flat on each $I \in \setI$ when translated by $\vec{v}$ is a $kr$-dimensional affine subspace.
Hence Lemma \ref{lem:incoherence} implies that
\[
  \sup_{f \in S_\setI} \frac{\left| \langle \vec{\epsilon}, \vec{f}_I+ \vec{v}_I \rangle \right|}{\| \vec{f}_I + \vec{v}_I \|} \; \leq \; O \left( \sigma \sqrt{k r + \log \frac{1}{\delta'}} \right)
\]
with probability at least $1 - \delta'$.
A basic combinatorial argument shows that there are $\binom{n + k - 1}{k - 1} = n^{O(k)}$ such different partitions $\setI$.
Let $\delta' = \delta / \binom{n - k}{k}$.
Then the result follows by union bounding over all the different possible partitions.
\end{proof}

\subsection{Runtime of linear least squares}
The appeal of linear least squares is not only its statistical properties, but also that the estimator can be computed efficiently.
In our algorithm, we invoke linear least squares multiple times as a black-box subrountine.
Primarily, we use the following theorem:
\begin{fact}
Let $\vec{A} \in \R^{n \times d}$ be an arbitrary data matrix, and let $\vec{y}$ be the set of responses.
Then there is an algorithm $\textsc{LeastSquares}(A, y)$ that runs in time $O(n d^2)$ and computes the least squares fit to this data.
\end{fact}
There has been work on faster, approximate algorithms that would suffice for our purposes in theory.
These algorithms offer a better dependence on the dimension $d$ in exchange for slightly more complicated approximation guarantees and somewhat more complicated algorithms (for instance, see \cite{CW13}).
However, the classical algorithms for least squares with time complexity $O(n d^2)$ are more commonly used in practice.
For this reason and to simplify our exposition, we thus present our results using the running time of the classical algorithms for least squares regression.
Specifically, we write $\textsc{LeastSquares}(\vec{X}, \vec{y}, I)$ to denote an algorithm that computes a least squares fit on a given interval $I \subseteq [n]$ and assume that $\textsc{LeastSquares}(\vec{X}, \vec{y}, I)$ runs in time $O(|I| \cdot d^2)$ for all $I \subseteq [n]$.

\section{Finding the least squares estimator via dynamic programming}
In this section, we first present a dynamic programming approach (DP) to piecewise linear regression.
We do not believe these results to be novel, but to the best of our knowledge, these results appear primarily as folklore in the literature.
For completeness, we demonstrate the fastest DP we are aware of, and we also prove its statistical guarantees.
Not only will this serve as a good warm-up for the later, more complex proofs, but we will also need a variant of this result in the later analyses.
\subsection{The exact DP}
We first describe the exact dynamic program.
It will simply find the $k$-piecewise linear function which minimizes the sum-squared error to the data.
In other words, it outputs
\[
  \argmin_{f \in \mathcal{L}_k} \sum_{i = 1}^n \left(y_i - f(\vec{x}_i) \right)^2 \quad  = \quad \argmin_{f \in \mathcal{L}_k} \norm{\vec{y} - \vec{f}}^2 \; ,
\]
which is simply the least-squares fit to the data amongst all $k$-piecewise linear functions.
The dynamic program computes the estimator $f$ as follows.
For $i = 1, \ldots, n$ and $j = 1, \ldots, k$, let $A(i, j)$ denote the best error achievable by a $j$-piecewise linear function on the data pairs $\{ (\vec{x}_\ell, y_\ell) \}_{\ell = 1}^i$, that is, the best error achievable by $j$ pieces for the first $i$ data points.
Then it is not hard to see that 
\[
  A(i, j) = \min_{i' < i} \left( \mathrm{err} (\vec{X}, \vec{y}, \{i' + 1, \ldots, i \}) + A(i', j - 1) \right) \; ,
\]
where for any interval $I$, we let $\mathrm{err} (\vec{X}, \vec{y}, I)$ denote the sum-squared error to the data of the best least squares fit to the data points $\{ (\vec{x}_\ell, y_\ell) \}_{\ell \in I}$.
That is, if $\fls_I$ is the least squares fit to the data in $I$, then $ \mathrm{err} (\vec{X}, \vec{y}, I) = \norm{\vec{y}_I - \vecfls_I}^2$.

The algorithm then uses dynamic programming to fill out this $n \times k$ sized table of $A$ values, starting at $i = 1$ and $j = 1$.
After having done so, the algorithm does one additional pass backwards over the table to actually find the path through the table which achieves the best error.
We can optimize this by first computing the error quantities $\mathrm{err} (\vec{X}, \vec{y}, \{i' + 1, \ldots, i \})$ for all $i' < i$, and then using a lookup table to find their values while actually executing the DP.
Given such a lookup table, the DP runs in time $O(n^2 k)$.
Naively, the construction of this look-up table would take time which is $O(n^3 d^2)$ since there are $O(n^2)$ linear regression problems of size $O(n)$, each of which takes $O(n d^2)$ time to solve.

However, we can speed this up.
Consider a fixed interval $I \subset [n]$.
Then the least squares fit on this interval is of the form $\vec{X}_I^T \vec{X}_I \vec{\theta}_I = \vec{X}_I^T \vec{y}_I$.
Assuming the matrix $\vec{X}_I^T \vec{X}_I$ is invertible, we have that $\vec{\theta}_I = \left( \vec{X}_I^T \vec{X}_I \right)^{-1} \vec{X}_I^T \vec{y}_I$.
Moreover, the error of the fit is given by
\begin{align*}
\left\| \vec{X}_I \vec{\theta}_I - \vec{y}_I \right\|^2 &= \left( \vec{X}_I \vec{\theta}_I - \vec{y}_I \right)^T \left( \vec{X}_I \vec{\theta}_I - \vec{y}_I \right) \\
&= \vec{\theta}_I^T \vec{X}_I^T \vec{X}_I \vec{\theta}_I - 2 \vec{y}^T \vec{X}_I \vec{\theta}_I + \vec{y}_I^T \vec{y}_I \\
&= \vec{y}_I^T \vec{X}_I \left[ \left( \vec{X}_I^T \vec{X}_I \right) \right]^{-1} \vec{X}_I^T \vec{y}_I - 2 \vec{y}_I^T \vec{X}_I \vec{\theta}_I + \vec{y}_I^T \vec{y_I} \; .
\end{align*}
The main point is the following: given $\left( \vec{X}_I^T \vec{X}_I \right)^{-1}$ and $\vec{X}_I^T \vec{y}_I$, we can compute all the remaining quantities in time which is $O(d^2)$.
To compute $\left( \vec{X}_I^T \vec{X}_I \right)^{-1}$ and $\vec{X}_I^T \vec{y}_I$, we use the fact that our regression instances are not arbitrary but very closely related.
As a result, we can re-use computation from previous regression instances that we have already solved in order to speed up later regression instances.

Concretely, the calculations above indicate that we need to compute $\vec{X}_I^T \vec{y}_I$ for all intervals $I \subseteq [n]$.
But these are all just sub-vectors of the vector $\vec{X}^T \vec{y}$, which we compute once in time $O(n d)$ and then use later on in all of the remaining computations.
More non-trivially, we also need to compute $\left( \vec{X}_I \vec{X}_I \right)^{-1}$ for all $I \subseteq [n]$.
However, observe that if we let $I = \{\ell, \ldots, p - 1\}$ for $\ell < p$, then $\vec{X}_{I \cup \{p\}} ^T \vec{X}_{I \cup \{p\}} = \vec{X}_I^T \vec{X}_I + \vec{x}_p \vec{x}_p^T$, so that adding a single data point to the data matrix corresponds to a rank-one update of the data matrix.
Moreover, the effect of a rank-one update to the inverse of the matrix is well-known:
\begin{theorem}[Sherman-Morrison formula]
Suppose $\vec{M} \in \R^{d \times d}$ is an invertible square matrix, and let $\vec{v} \in \R^d$ be \newjd{such} that $1 + \vec{v}^T \vec{M}^{-1} \vec{v} \neq 0$.
Then
\[
\left( \vec{M} + \vec{v} \vec{v}^T \right)^{-1} = \vec{M}^{-1} - \frac{\vec{M}^{-1} \vec{v} \vec{v}^T \vec{M}^{-1}}{1 + \vec{v}^T \vec{M}^{-1} \vec{v}} \; .
\]
Thus, given $\vec{M}^{-1}$ and $\vec{v}$ satisfying these conditions, we may compute $\left( \vec{M} + \vec{v} \vec{v}^T \right)^{-1}$ in time $O(d^2)$.
\end{theorem}
Therefore, the dynamic program can do as follows: for each $\ell = 1, \ldots, n$, first compute $\left(  \vec{X}_I^T \vec{X}_I \right)^{-1}$ for the interval $I$ of length $d$ starting at $\ell$, and then use the Sherman-Morrison update formula\footnote{In general, our matrices may not be invertible, or the update vector may not satisfy the condition in the Sherman-Morrison formula, but in practice it seems these issues don't come up and thus we do not worry about them here.
In general there are more complicated formulas for rank one updates for pseudo-inverses here but we will not cover them for simplicity.} 
to compute $\left( \vec{X}_I^T \vec{X}_I \right)^{-1}$ for every interval $I$ starting at $\ell$ in total time $O(n d^2)$.
Thus the algorithm takes $O(n^2 d^2)$ time total to compute all of the $\left( \vec{X}_I^T \vec{X}_I \right)^{-1}$.
As demonstrated earlier, this implies the following:
\begin{theorem}
\label{thm:time-DP}
The exact dynamic program runs in time $O(n^2 (d^2 + k))$.
\end{theorem}
\subsection{Error analysis for the exact DP}
We now turn our attention to the learning rate of the exact DP.
We show:
\begin{theorem}
\label{thm:DP}
Let $\delta > 0$, and let $\fls$ be the $k$-piecewise linear estimator returned by the exact DP.
Then, with probability $1 - \delta$, we we have that
\[
  \MSE (\fls) \; \leq \; O\left(\sigma^2 \, \frac{kr + k \log \frac{n}{\delta}}{n}\right) \; .
\]
\end{theorem}
\begin{proof}
We follow the proof technique for convergence of linear least squares as presented in \cite{Rig15}.
Recall $f$ is the true $k$-piecewise linear function.
Then, by the definition of the least squares fit, we have that
\[
\norm{\vec{y} - \vecfLS}^2 \; \leq \; \norm{\vec{y} - \vec{f}}^2 \; = \; \norm{\vec{\epsilon}}^2 \; .
\]
By expanding out $\norm{\vec{y} - \vecfLS}^2$, we have that
\begin{align*}
\norm{\vec{y} - \vecfLS}^2 \; &= \; \norm{\vec{f} + \vec{\epsilon} - \vecfLS}^2 \\
&= \; \norm{\vecfLS - \vec{f}}^2 + 2 \, \langle \vec{\epsilon}, \vec{f}- \vecfLS \rangle + \norm{\vec{\epsilon}}^2 \ . \numberthis \label{eq:DP-LHS}
\end{align*}
From these two calculations, we gather that 
\begin{align*}
\norm{\vecfLS - \vec{f}}^2 \; &\leq \; 2 \, \langle \vec{\epsilon}, \vec{f}- \vecfLS\rangle \\
&\leq \; O \left( \sigma \sqrt{kr + k \log \frac{n}{\delta}} \right) \cdot \| \vecfLS - \vec{f} \| \; ,
\end{align*}
with probability $1 - \delta$, where the last line follows from Corollary \ref{cor:sim-incoherence}.
A simple algebraic manipulation  from this last inequality yields the desired statement.
\end{proof}

The same proof technique can also be easily adapted to yield the following slight extension of Theorem \ref{thm:DP}, which can be proven via a union bound over all sets of $k$ disjoint intervals.
We omit a proof here for conciseness.
\begin{lemma}
\label{lem:flat}
Fix $\delta > 0$.
Then with probability $1 - \delta$ we have the following:
for all disjoint sets of $k$ intervals $I_1, \ldots, I_k$ of $[n]$ so that $f$ is flat on each $I_\ell$, the following inequality holds:
\[
  \sum_{\ell = 1}^k \norm{\vecfLS_{I_\ell} - \vec{f}_{I_\ell}}^2 \; \leq \; O (\sigma^2 \, k (r + \log (n / \delta))) \; .
\]
\end{lemma}

\section{A simple greedy merging algorithm}
In this section, we give a novel greedy algorithm which runs much faster than the DP, but which achieves a somewhat worse learning rate.
However, we show both theoretically and experimentally that the tradeoff between speed and statistical accuracy for this algorithm is markedly better than it is for the exact DP.

\subsection{The greedy merging algorithm}
The overall structure of the algorithm is quite similar to \cite{ADHLS15}, however, the merging criterion is different, and as explained above, the guarantees proved in that paper are insufficient to give non-trivial learning guarantees for regression.

Our algorithm here also requires an additional input $s^2$, which is defined to be the variance of the $\epsilon_i$ variables, i.e., $s^2 = \EE [\epsilon_i^2]$.
Requiring that we know $s^2$ is a drawback, and in Section \ref{sec:bucketing} we give a slightly more complicated algorithm which does not require knowledge of $s$.

We give the formal pseudocode for the \newjd{procedure} in Algorithm~\ref{alg:greedy}.
In the pseudocode we provide two additional tuning parameters $\tau, \gamma$.
This is because in general our algorithm cannot provide a $k$-histogram, but instead provides an $O(k)$ histogram, which for most practical applications suffices.
The tuning parameters allow us to trade off running time for fewer pieces.
In the typical use case we will have $\tau, \gamma = \Theta(1)$, in which case our algorithm will output an $O(k)$-piecewise linear function in time $O(n d^2 \log n)$ time.

\begin{algorithm*}[htb]
\begin{algorithmic}[1]
\Function{GreedyMerging}{$\tau, \gamma, s, \vec{X}, \vec{y}$}

\LineComment{\emph{Initial partition of $[n]$ into intervals of length $1$.}}
\State $\setI^0 \gets \{\{ 1 \}, \{ 2 \}, \ldots, \{ n \}\}$ 

\LineComment{\emph{Iterative greedy merging (we start with $j \gets 0$).}}
\While{$| \setI^j | > (2 + \frac{2}{\tau}) k + \gamma$}
  \State Let $s_j$ be the current number of intervals.
  \LineComment{\emph{Compute the least squares fit and its error for merging neighboring pairs of intervals.}}
  \For{$u \in \{1, 2, \ldots, \frac{s_j}{2}\}$}
    \State $\vec{\theta}_u \gets \textsc{LeastSquares}(\vec{X}, \vec{y}, I_{2u - 1} \cup I_{2u})$
    \State $e_u = \| \vec{y}_I - \vec{X}_I \vec{\theta}_u \|_2^2 - s^2 |I_{2u - 1} \cup I_{2u}|$
  \EndFor
  \State Let $L$ be the set of indices $u$ with the $(1 + \frac{1}{\tau})k$ largest errors $e_u$,
  \Statex $\qquad$ breaking ties arbitrarily. 
  \State Let $M$ be the set of the remaining indices. 
  \LineComment{\emph{Keep the intervals with large merging errors.}}
  \State $\setI^{j+1} \gets \bigcup\limits_{u \in L} \{I_{2u - 1}, I_{2u}\}$ 
  \LineComment{\emph{Merge the remaining intervals.}}
  \State $\setI^{j+1} \gets \setI^{j+1} \cup \{I_{2u - 1} \cup I_{2u} \, | \, u \in M \}$ 
  \State $j \gets j + 1$
\EndWhile
\State \textbf{return} the least squares fit to the data on every interval in $\setI^j$

\EndFunction
\end{algorithmic}
\caption{Piecewise linear regression by greedy merging.}
\label{alg:greedy}
\end{algorithm*}

\subsection{Runtime of \textsc{GreedyMerging}}
In this section we prove that our algorithm has the following, nearly-linear running time.
The analysis is similar to the analysis presented in \cite{ADHLS15}.
\begin{theorem}
\label{thm:time-greedy}
Let $\vec{X}$ and $\vec{y}$ be as in~\eqref{eq:model}. Then $\textsc{GreedyMerging}(\tau, \gamma, s, \vec{X}, \vec{y})$ outputs a $\left( (2 + \frac{2}{\tau}) k + \gamma \right)$-piecewise linear function and runs in time $O(nd^2 \log(n / \gamma))$.
\end{theorem}
Before we prove this theorem, we compare this with the running time for the exact DP as given in Theorem \ref{thm:time-DP}.
Our main advantage is that our runtime scales linearly with $n$ instead of quadratically. This manifests itself as a substantially win theoretically in most reasonable regimes, and also as a big win in practice---see our experiments for more details there.
 
\begin{proof}[Proof of Theorem \ref{thm:time-greedy}]
We first bound the time it takes to run any single iteration of the algorithm.
In any iteration $j$, we do a linear least squares regression problem on each interval $I \in \setI^j$; each such problem takes $O(|I| d^2)$ time; hence solving them all takes $O(n d^2)$ time.
Computing the $e_u$ given the least squares fit takes no additional time asymptotically,
and finding the $\left( 1 + \frac{1}{\tau}\right)$ largest errors takes linear time \cite{CLRS09}.
Afterwards the remaining computations in this iteration can easily be seen to be done in linear time.
Hence each iteration takes $O(n d^2)$ time to complete.

We now bound the number of iterations of the algorithm.
By the same analysis as that done in the proof of Theorem 3.4 in \cite{ADHLS15} one can show that the algorithm terminates after at most $\log (n / \gamma)$ iterations.
Thus the whole algorithm runs in time $O(nd^2 \log(n / \gamma))$ time, as claimed.
\end{proof}

\subsection{Analysis of \textsc{GreedyMerging}}

\begin{theorem}
\label{thm:greedy}
Let $\delta > 0$, and let $\fhat$ be the estimator returned by \textsc{GreedyMerging}.
Let $m = (2 + \frac{2}{\tau}) k + \gamma$ be the number of pieces in $\fhat$.
Then, with probability $1 - \delta$, we have that
\[\MSE (\fhat) \leq O\left( \sigma^2 \left( \frac{m (r + \log (n / \delta) )}{n} \right) + \sigma \frac{\tau + \sqrt{k}}{\sqrt{n}} \log \left( \frac{n}{\delta} \right) \right) \; .\]
\end{theorem}

\begin{proof}
We first condition on the event that Corollaries \ref{cor:unif-err-bound}, \ref{cor:int-incoherence} and \ref{cor:sim-incoherence}, and Lemma \ref{lem:flat} all hold with error parameter $O(\delta)$, so that together they all hold with probability at least $1 - \delta$.
Let $\setI = \{ I_1, \ldots, I_{m} \}$ be the final partition of $[n]$ that our algorithm produces.
Recall $f$ is the ground truth $k$-piecewise linear function.
We partition the intervals in $\setI$ into two sets:  
\begin{align*}
\mathcal{F} &= \{I \in \setI: \mbox{$f$ is flat on $I$} \} \; ,\\
\mathcal{J} &= \{I \in \setI: \mbox{$f$ has a jump on $I$} \} \; .
\end{align*}
We first bound the error over the intervals in $\mathcal{F}$.
By Lemma \ref{lem:flat}, we have
\begin{equation}
\label{eq:flat}
\sum_{I \in \mathcal{F}} \norm{\vec{f}_I - \vec{\fhat}_I}^2 \; \leq \; O (\sigma^2 |\mathcal{F}| (r + \log ( n / \delta))) \;,
\end{equation}
with probability at least $1 - O(\delta)$.

We now turn our attention to the intervals in $\mathcal{J}$ and distinguish two further cases.
We let $\mathcal{J}_1$ be the set of intervals in $\mathcal{J}$ which were never merged, and we let $\mathcal{J}_2$ be the remaining intervals.
If the interval $I \in \mathcal{J}_1$ was never merged, the interval contains one point, call it $i$.
Because we may assume that $\vec{x}_i \neq 0$, we know that for this one point, our estimator satisfies $\fhat (\vec{x}_i) = y_i$, since this is clearly the least squares fit for a linear estimator on one nonzero point.
Hence Corollary \ref{cor:unif-err-bound} implies that the following inequality holds with probability at least $1 - O(\delta)$:
\begin{align*}
\sum_{I \in \mathcal{J}_1} \norm{\vec{f}_I - \vec{\fhat}_I}^2 &= \sum_{I \in \mathcal{J}_1} \norm{\epsilon_I}^2 \\
&\leq \sigma^2 \left( \sum_{I \in \mathcal{J}_1} |I| + O \left( \log \frac{n}{\delta} \right) \sqrt{\sum_{I \in \mathcal{J}_1} |I|} \right) \\
&\leq \sigma^2 \left( m + O \left( \log \frac{n}{\delta} \right) \sqrt{m} \right) \; . \numberthis \label{eq:jump1}
\end{align*}

We now finally turn our attention to the intervals in $\mathcal{J}_2$.
Fix an interval $I \in \mathcal{J}_2$.
By definition, the interval $I$ was merged in some iteration of the algorithm.
This implies that in that iteration, there were $(1 + 1 / \tau) k$ intervals $M_1, \ldots, M_{(1 + 1/ \tau) k}$ so that for each interval $M_\ell$, we have
\begin{equation}
\label{eq:merge-cond}
\norm{\vec{y}_I - \vec{\fhat}_I}^2 - s^2 |I| \; \leq \; \norm{\vec{y}_{M_\ell} - \vecfLS_{M_\ell}}^2 - s^2 |M_\ell| \; .
\end{equation}
Since the true, underlying $k$-piecewise linear function $f$ has at most $k$ jumps, this means that there are at least $k / \tau$ intervals of the $M_\ell$ on which $f$ is flat.
WLOG assume that these intervals are $M_1, \ldots, M_{k / \tau}$.

Fix any $\ell = 1, \ldots, k / \tau$.
Expanding out the RHS of~\eqref{eq:merge-cond} using the definition of $y_i$ gives
\begin{align*}
\left\| \vec{y}_{M_\ell} - \vecfLS_{M_\ell} \right\|^2 - s^2 |M_\ell| &= \left\| \vec{f}_{M_\ell} - \vecfLS_{M_\ell} \right\|^2 + 2 \langle \vec{\epsilon}_{M_\ell}, \vec{f}_{M_\ell} - \vecfLS_{M_\ell} \rangle + \| \vec{\epsilon}_{M_\ell} \|^2 - s^2 |M_\ell| \\
&=  \left\| \vec{f}_{M_\ell} - \vecfLS_{M_\ell} \right\|^2 + 2 \langle \vec{\epsilon}_{M_\ell}, \vec{f}_{M_\ell} - \vecfLS_{M_\ell} \rangle + \sum_{i \in M_\ell} (\epsilon_i^2 - s^2) \; .
\end{align*}
Thus, we have that in aggregate,
\begin{align*}
\sum_{\ell = 1}^{k / \tau} \left\| \vec{y}_{M_\ell} - \vecfLS_{M_\ell} \right\|^2 - s^2 |M_\ell| & =  \sum_{\ell = 1}^{k / \tau} \left\| \vec{f}_{M_\ell} - \vecfLS_{M_\ell} \right\|^2 + 2 \sum_{\ell = 1}^{k / \tau} \langle \vec{\epsilon}_{M_\ell}, \vec{f}_{M_\ell} - \vecfLS_{M_\ell} \rangle + \sum_{\ell = 1}^{k / \tau}\sum_{i \in M_\ell} (\epsilon_i^2 - s^2) \; . \numberthis \label{eq:flat-allterms}
\end{align*}
We will upper bound each term on the RHS in turn.
First, since the function $f$ is flat on each $M_\ell$ for $\ell = 1, \ldots, k / \tau$, Lemma \ref{lem:flat} implies
\begin{equation}
\label{eq:flat-term1}
\sum_{\ell = 1}^{k / \tau} \left\| \vec{f}_{M_\ell} - \vecfLS_{M_\ell} \right\|^2 \leq O \left(\sigma^2 \frac{k}{\tau} \left(r + \log \frac{n}{\delta} \right) \right) \; , \end{equation}
with probability at least $1 - O(\delta)$.

Moreover, note that the function $\fls_{M_\ell}$ is a linear function on $M_\ell$ of the form $\fls_{M_\ell}(\vec{x}) = \vec{x}^T \hat{\beta}$, where $\hat{\beta} \in \R^d$ is the least-squares fit on $M_{\ell}$.
Because $\vec{f}$ is just a fixed vector, the vector $\vec{f}_{M_\ell} -\vecfLS_{M_\ell} $ lives in the affine subspace of vectors of the form $\vec{f}_{M_\ell} + (\vec{X}_{M_\ell}) \eta$ where $\eta \in \R^d$ is arbitrary.
So Corollary \ref{cor:sim-incoherence} and~\eqref{eq:flat-term1} imply that
\begin{align*}
\sum_{\ell = 1}^{k / \tau}\langle \vec{\epsilon}_{M_\ell}, \vec{f}_{M_\ell} - \vecfLS_{M_\ell} \rangle &\leq \sqrt{\sum_{\ell = 1}^{k / \tau}\langle \vec{\epsilon}_{M_\ell}, \vec{f}_{M_\ell} - \vecfLS_{M_\ell} \rangle} \cdot \sup_{\eta} \frac{|\langle \vec{\epsilon}_{M_\ell},  \vec{X} \eta \rangle |}{\| \vec{X} \eta \|} \\
&\leq O \left(\sigma^2 \frac{k}{\tau} \left(r + \log \frac{n}{ \delta} \right) \right)\; . \numberthis \label{eq:flat-term2}
\end{align*}
with probability $1 - O(\delta)$.

By Corollary \ref{cor:unif-err-bound}, we get that with probability $1 - O(\delta)$,
\[\sum_{\ell = 1}^{k / \tau} \left( \sum_{i \in M_\ell} \epsilon_i^2 - s^2 |M_\ell| \right) \leq O\left( \sigma \log \frac{n}{\delta} \right) \sqrt{n} \; .\]
Putting it all together, we get that
\begin{equation}
\label{eq:flat-term3} 
\sum_{i = 1}^{k / \tau} \left( \left\| \vec{y}_{M_\ell} - \vecfLS_{M_\ell} \right\|^2 - s^2 |M_\ell|  \right) \leq O \left(\frac{k}{\tau} \sigma^2 \left(r+ \log \frac{n}{\delta} \right) \right) + O\left(\sigma \log \frac{n}{\delta} \right) \sqrt{n} \;  
\end{equation}
with probability $1 - O(\delta)$.
Since the LHS of~\eqref{eq:merge-cond} is bounded by each individual summand above, this implies that the LHS is also bounded by their average, which implies that
\begin{align*}
\norm{\vec{y}_I - \vec{\fhat}_I}^2 - s^2 |I| \; &\leq \; \frac{\tau}{k} \sum_{i = 1}^{k / \tau} \left( \norm{\vec{y}_{M_\ell} - \vecfls_{M_\ell}}^2 - s^2 |M_\ell| \right) \\
&\leq  \; O\left(\sigma^2 \left(r + \log \left( \frac{n}{\delta} \right) \right) \right) + O\left(\sigma \log \frac{n}{\delta} \right) \frac{\tau \sqrt{n}}{k} \; . \numberthis \label{eq:upper-bound}
\end{align*}
We now similarly expand out the LHS of~\eqref{eq:merge-cond}:
\begin{align*}
\norm{\vec{y}_I - \vec{\fhat}_I}^2 - s^2 |I| \; &= \; \norm{\vec{f}_I + \vec{\epsilon}_I - \vec{\fhat}_I}^2 - s^2 |I| \\
&= \; \norm{\vec{f}_I - \vec{\fhat}_I}^2 + 2 \langle \vec{\epsilon}_I, \vec{f}_I - \vec{\fhat}_I \rangle + \norm{\vec{\epsilon}_I}^2 - s^2 |I|  \; . \numberthis \label{eq:jump-allterms}
\end{align*}
From this we are interested in obtaining an upper bound on $\sum_{i \in I} (f (\vec{x}_i) - \fhat (\vec{x}_i))^2$, hence we seek to lower bound the second and third terms of~\eqref{eq:jump-allterms}.
The calculations here will very closely mirror those done above.

By Corollary \ref{cor:int-incoherence}, we have that
\[ 
2 \langle \vec{\epsilon}_I, \vec{f}_I - \vec{\fhat}_I \rangle \; \geq \; - O \left(\sigma \sqrt{r + \log \left( \frac{n}{\delta} \right)} \right) \norm{\vec{f}_I - \vec{\fhat}_I} \; ,
\]
and by Corollary \ref{cor:unif-err-bound} we have
\[
\norm{\vec{\epsilon}_I}^2 - s^2 |I| \; \geq \; - O \left(\sigma \log \frac{n}{\delta} \right) \sqrt{|I|} \; ,
\]
and so
\begin{equation}
\norm{\vec{y}_I - \vec{\fhat}_I}^2 - s^2 |I|  \; \geq \; \norm{\vec{f}_I - \vec{\fhat}_I}^2 - O \left(\sigma \sqrt{r + \log \left( \frac{n}{\delta} \right)} \right) \norm{\vec{f}_I - \vec{\fhat}_I} - O\left(\sigma \log \frac{n}{\delta} \right) \sqrt{|I|} \; . \label{eq:lower-bound}
\end{equation}
Putting~\eqref{eq:upper-bound} and~\eqref{eq:lower-bound} together yields that with probability $1 - O(\delta)$,
\begin{align*}
\norm{\vec{f}_I - \vec{\fhat}_I}^2 \; &\leq \; O\left( \sigma^2 \left(r + \log \left( \frac{n}{\delta} \right) \right) \right) \\
&+ O \left(\sigma \sqrt{r + \log \left( \frac{n}{\delta} \right)} \right) \norm{\vec{f}_I - \vec{\fhat}_I}  + O\left( \sigma \log \frac{n}{\delta} \right) \left( \frac{\tau \sqrt{n}}{k} + \sqrt{|I|} \right) \; .
\end{align*}

Letting $z^2 = \norm{\vec{\fhat}_I - \vec{f}_I}^2$, then this inequality is of the form $z^2 \leq b z + c$ where $b, c \geq 0$.
In this specific case, we have that
\begin{align*}
b &= O \left( \sigma \sqrt{r + \log \frac{n}{\delta} } \right) \; ,~ \mbox{and} \\
c &= O \left( \sigma^2 \left(r + \log \left( \frac{n}{\delta} \right) \right) \right) + O\left( \sigma \log \frac{n}{\delta} \right) \left( \frac{\tau \sqrt{n}}{k} + \sqrt{|I|} \right)
\end{align*}
We now prove the following lemma about the behavior of such quadratic inequalities:
\begin{lemma}
\label{lem:quadratic}
Suppose \newjd{$z^2 \leq bz + c$} where $b, c \geq 0$.
Then $z^2 \leq O(b^2 + c)$.
\end{lemma}
\begin{proof}
From the quadratic formula, the inequality implies that \newjd{$z \leq \frac{b + \sqrt{b^2 + 4c}}{2}$}.
From this, it is straightforward to demonstrate the desired claim.
\end{proof}
Thus, from the lemma, we have
\[
\norm{\vec{f}_I - \vec{\fhat}_I}^2 \; \leq \; O\left( \sigma^2 \left(r + \log \left( \frac{n}{\delta} \right) \right) \right) + O\left( \sigma \log \frac{n}{\delta} \right) \left( \frac{\tau \sqrt{n}}{k} + \sqrt{|I|} \right) \; .
\]

Hence the total error over all intervals in $\mathcal{J}_2$ can be bounded by:
\begin{align*}
\sum_{I \in \mathcal{J}_2} \norm{\vec{f}_I - \vec{\fhat}_I}^2 \;  &\leq \; O \left(k \sigma^2 \left(r + \log \left( \frac{n}{\delta} \right) \right) \right) + O\left( \sigma \log \frac{n}{\delta} \right) \left( \tau \sqrt{n} + \sum_{I \in \mathcal{J}}\sqrt{|I|} \right) \\
&\; \leq O \left( k \sigma^2 \left(r + \log \left( \frac{n}{\delta'} \right) \right) \right) + O\left(\sigma \log \frac{n}{\delta} \right) \left( \tau \sqrt{n} + \sqrt{k n} \right) \numberthis \label{eq:jump2} \; .
\end{align*}
In the last line we use that the intervals $I \in \mathcal{J}_2$ are disjoint (and hence their cardinalities sum up to at most $n$), and that there are at most $k$ intervals in $\mathcal{J}_2$ because the function $f$ is $k$-piecewise linear.
Finally, applying a union bound and summing~\eqref{eq:flat},~\eqref{eq:jump1}, and~\eqref{eq:jump2} yields the desired conclusion.
\end{proof}

\section{A variance-free merging algorithm}
\label{sec:bucketing}
In this section, we give a variant of the above algorithm \newjd{that} does not require knowledge of the noise variance $s^2$.
The formal pseudo code is given in Algorithm \ref{alg:bucket}.

\begin{algorithm*}[htb]
\begin{algorithmic}[1]
\Function{BucketGreedyMerging}{$\gamma, \vec{X}, \vec{y}$}

\LineComment{Initial histogram.}
\State Let $\setI^0 \gets \{\{ 1 \}, \{ 2 \}, \ldots, \{ n \}\}$ be the initial partition of $[n]$ into intervals of length $1$.

\LineComment{Iterative greedy merging (we start with $j = 0$).}
\While{$| \setI^j | > (2(k + 1) + \gamma) \log n$}
  \State Let $s_j$ be the current number of intervals.
  \LineComment{Compute the least squares fit and its error for merging neighboring pairs of intervals.}
  \For{$u \in \{1, 2, \ldots, \frac{s_j}{2}\}$}
     \State $I'_u \gets I_{2u - 1} \cup I_{2u}$
    \State $\vec{\theta}_u \gets \textsc{LeastSquares}(\vec{X}, \vec{y}, I'_u)$
    \State $e_u = \frac{1}{|I'_u|} \| \vec{y}_I - \vec{X}_I \vec{\theta}_u \|_2^2 $
  \EndFor
  \For{$\alpha = 0, \ldots, \log (n) - 1$}
  	\State Let $B_\alpha$ be the set of indices $u$ so that $2^\alpha \leq |I'_u| \leq 2^{\alpha + 1}$
	\State Let $L_\alpha$ be the set of indices $u \in B_\alpha$ with the $k + 1$ largest $e_u$ amongst $u \in B_\alpha$,
  \Statex $\qquad$ breaking ties arbitrarily.
	\State Let $M_\alpha$ be the set of the remaining indices $u$ in $B_\alpha$.
  	\LineComment{\emph{Keep the intervals in each $B_\alpha$ with large merging errors.}}
  	\State $\setI^{j+1} \gets \bigcup\limits_{u \in L_\alpha} \{I_{2u - 1}, I_{2u}\}$ 
  	\LineComment{\emph{Merge the remaining intervals in each $B_\alpha$.}}
  	\State $\setI^{j+1} \gets \setI^{j+1} \cup \{I'_u \, | \, u \in M_\alpha \}$ 
  \EndFor
  \State $j \gets j + 1$
\EndWhile
\State \textbf{return} the least squares fit to the data on every interval in $\setI^j$

\EndFunction
\end{algorithmic}
\caption{Variance-free greedy merging with bucketing.}
\label{alg:bucket}
\end{algorithm*}

We now state the running time for \textsc{BucketGreedyMerge}.
The running time analysis for \textsc{BucketGreedyMerge} is almost identical to that of \textsc{GreedyMerge}; hence we omit its proof.
\begin{theorem}
\label{thm:time-bucket}
Let $\vec{X}$ and $\vec{y}$ be as in~\eqref{eq:model}. Then $\textsc{BucketGreedyMerging}(\gamma, \vec{X}, \vec{y})$ outputs a $\left(2 (k + 1) + \gamma \right) \log n$-piecewise linear function and runs in time $O(nd^2 \log(n / \gamma))$.
\end{theorem}

\subsection{Analysis of \textsc{BucketGreedyMerge}}
This section is dedicated to the proof of the following theorem:
\begin{theorem}
\label{thm:bucket}
Let $\fhat$ be the $m$-piecewise linear function that is returned by \textsc{BucketGreedyMerge}, where $m = \left(2 (k + 1) + \gamma \right) \log n$.
Then, with probability $1 - \delta$, we have
\[\MSE (\fhat) \leq O\left( \sigma^2 \left( \frac{m(r  + \log (n / \delta) )}{n} \right) + \sigma \sqrt{\frac{k}{n}} \log \left( \frac{n}{\delta} \right) \right) \; .\]
\end{theorem}
\begin{proof}
As in the proof of Theorem \ref{thm:greedy}, we let $\mathcal{I}$ be the final partition of $[n]$ that our algorithm produces.
We also similarly condition on the event that Corollaries \ref{cor:unif-err-bound} and \ref{cor:sim-incoherence} and Lemma \ref{lem:flat}  all hold with parameter $O(\delta)$.
We again partition $\setI$ into two sets $\mathcal{F}$ and $\mathcal{J}$ as before, and further subdivide $\mathcal{J}$ into $\mathcal{J}_1$ and $\mathcal{J}_2$.
The error on $\mathcal{F}$ and $\mathcal{J}_1$ is the same as the error given in~\eqref{eq:flat} and~\eqref{eq:jump1}.
The only difference is the error on intervals in $\mathcal{J}_2$.

Let $I \in \mathcal{J}_2$ be fixed.
By definition, this means there was some iteration and a collection of some $k + 1$ disjoint intervals $M_1, \ldots, M_{k + 1}$ so that $|M_\ell| / 2 \leq |I| \leq 2 |M_\ell|$ and
\[\frac{1}{|I|} \left\| \vec{y}_I - \vec{\fhat}_I \right\|^2 \leq \frac{1}{|M_\ell|} \left\| \vec{y}_{M_\ell} - \vecfls_{M_\ell} \right\|^2\]
for all $\ell = 1, \ldots, k + 1$.
Since $f$ has at most $k$ jumps, this means that there is at least one interval $M_\ell$ so that $f$ is flat on $M_\ell$.
WLOG assume that $f$ is flat on $M_1$.
By the same kinds of calculations done in the proof of Theorem \ref{thm:greedy}, we have that with probability $1 - \delta$,
\[ \left\| \vec{y}_{M_1} - \vecfls_{M_1} \right\|^2 \leq  O \left( \sigma^2 ( r + \log \frac{n}{\delta} ) \right) + \sigma^2 |M_1 | +  O(\sigma \log (n / \delta)) \sqrt{|M_1|} \]
and
\begin{align*}
\left\| \vec{y}_I - \vec{\fhat}_I \right\|^2 &\geq  \left\| \vec{f}_I - \vec{\fhat}_I \right\|^2 -  O \left( \sigma \sqrt{r + \log (n / \delta) } \right) \left\| \vec{y}_I - \vec{\fhat}_I \right\| + \sigma^2 |I| - O(\sigma \log (n / \delta)) \sqrt{|I|} \; ,
\end{align*}
and putting these two equations together and rearranging, we have
\begin{align*}
 \left\| \vec{f}_I - \vec{\fhat}_I \right\|^2 &\leq O \left( \sigma^2 (r + \log (n / \delta) ) \right) + \sigma^2 \cdot O(\log (n / \delta)) \left( \frac{|I|}{\sqrt{|M_\ell|}} + \sqrt{|I|} \right) \\
&~~~~ + O \left( \sigma \sqrt{r + \log (n / \delta) } \right) \left\| \vec{y}_I - \vec{\fhat}_I \right\|\\
&\leq  O \left( \sigma^2 (r+ \log (n / \delta) ) \right) + O\left( \sigma^2 \log (n / \delta) \sqrt{|I|} \right) + O \left( \sigma \sqrt{r + \log (n / \delta) } \right) \left\| \vec{y}_I - \vec{\fhat}_I \right\| \; ,
\end{align*}
where in the last inequality we used that $|M_\ell| \geq |I| / 2$.
By Lemma \ref{lem:quadratic}, this implies that
\[
\left\| \vec{f}_I - \vec{\fhat}_I \right\|^2 \leq O \left( \sigma^2 (r+ \log (n / \delta) ) \right) + O\left( \sigma^2 \log (n / \delta) \sqrt{|I|} \right) \; .
\]
Since there are at most $k$ elements in $\mathcal{J}_2$, and they are all disjoint, this yields that
\begin{equation}
\label{eq:bucketjump}
\sum_{I \in \mathcal{J}_2} \left\| \vec{f}_I - \vec{\fhat}_I \right\|^2 \leq O \left( k \sigma^2 \left( r + \log \frac{n}{\delta} \right) \right) + O\left( \sigma  \log \left( \frac{n}{\delta} \right) \sqrt{k n} \right)
\end{equation}
and putting this equation together with~\eqref{eq:flat} and~\eqref{eq:jump1} yields the desired conclusion.
\end{proof}

\subsection{Postprocessing}
One unsatisfying aspect of \textsc{BucketGreedyMerge} is that it outputs $O(k \log n)$ pieces. 
Not only does this increase the size of the representation nontrivially when $n$ is large, but it also increases the error rate: it is the reason why the first term in the error rate given in Theorem \ref{thm:bucket} has an additional $\log n$ factor as opposed the rate in Theorem \ref{thm:greedy}.
In this section, we give an efficient postprocessing procedure for \textsc{BucketGreedyMerge} which, when run on the output of \textsc{BucketGreedyMerge}, outputs an $O(k)$ histogram with the same rate as before.
In fact, the rate is slightly improved as we are able to remove the $\log n$ factor mentioned above.

The postprocessing algorithm \textsc{Postprocessing} takes as input a partition $\setI$ of $[n]$ of size $O(k \log n)$ and a target number of pieces $k$. 
It then performs the following steps: starting from the $O(k \log n)$ partition $\setI$, run the dynamic program (DP) on intervals with breakpoints amongst the breakpoints of $\setI$ to find the $2k + 1$ partition $\setI_p$ (whose endpoints are also endpoints of $\setI$) that minimizes the sum squared error to the data.
The running time analysis is identical to that of the DP with two exceptions: first, we only need to fill out a $O(k \log n) \times (2k + 1)$ size table (as compared to an $n \times k$ sized table).
Second, we are no longer performing rank one updates because we instead compute updates in large chunks, we cannot use the Sherman-Morrison formula to speed up the computation of the least-squares fits.
Hence, we obtain the following theorem:
\begin{theorem}
\label{thm:time-postprocessing}
Given a partition $\setI$ of $[n]$ into $O(k \log n)$ pieces, $\textsc{Postprocessing}(\setI, k)$ runs in time $\Otilde(k^3 d^2)$, and outputs a $(2k + 1)$-piecewise linear function, where the $\Otilde$ hides $\poly \log (n)$ factors.
\end{theorem}

We now show that the algorithm still provides the same (in fact, slightly better) statistical guarantees as the original partition:
\begin{theorem}
\label{thm:postprocessing}
Fix $\delta > 0$.
Let $\fhat^p$ be the estimator output by \textsc{PostprocessedBucketGreedyMerge}.
Then, with probability $1 - \delta$, we have
\[
\MSE (\fhat^p) \leq O \left( \sigma^2 \frac{k \left( r + \log \frac{n}{\delta} \right)}{n} + \sigma  \log \left( \frac{n}{\delta} \right) \sqrt{\frac{k}{n}} \right)
 \; .\]
\end{theorem}
\begin{proof}
Let $\setI$ and $\mathcal{J}$ be as in the proof of Theorem \ref{thm:bucket}.
Let $\setI_p = \{J_1, \ldots, J_{2k + 1}\}$ be the intervals in the partition that we return.
We again condition on the event that Corollaries \ref{cor:unif-err-bound} and \ref{cor:sim-incoherence} and Lemma \ref{lem:flat} all hold with parameter $O(\delta)$.
The following will then all hold with probability $ 1- \delta$.

Define the partition $\mathcal{K}$ to be the partition that contains every interval in $\mathcal{J}$ and exactly one interval between any two non-consecutive intervals in $\mathcal{J}$ (i.e., $\mathcal{K}$ merges all flat intervals).
Moreover, let $g$ be the $(2k + 1)$-piecewise linear function which is the least squares fit to the data on each interval in $I \in \mathcal{K}$.
This is clearly a possible solution for the dynamic program, and therefore we have
\begin{align*}
\norm{\vec{\fhat} - \vec{y}}^2 \;  &\leq \; \norm{\vec{g} - \vec{y}}^2 \\
&= \sum_{I \in \mathcal{K} \setminus \mathcal{J}} \norm{\vec{g}_I - \vec{y}_I}^2 + \sum_{I \in \mathcal{J}} \norm{\vec{g}_I -\vec{y}_I}^2 \; . \numberthis \label{eq:postDP-opt}
\end{align*}
We will expand and then upper bound the RHS of~\eqref{eq:postDP-opt}.
First, by calculations similar to those employed in the proof of Theorem \ref{thm:greedy}, we have that
\[
  \sum_{I \in \mathcal{K} \setminus \mathcal{J}} \norm{\vec{g}_I - \vec{y}_I}^2 \; \leq \; O\left( k \sigma^2 \left( r + \log \frac{n}{\delta} \right) \right) +  \sum_{I \in \mathcal{K} \setminus \mathcal{J}} \norm{\vec{\epsilon}_I}^2  \; ,
\]
and from~\eqref{eq:bucketjump} and Corollary~\ref{cor:sim-incoherence} we have
\begin{align*}
\sum_{I \in \mathcal{J}}  \norm{\vec{g}_I - \vec{y}_I}^2 \;  &= \; \sum_{I \in \mathcal{J}} \norm{\vec{f}_I + \vec{\epsilon}_I - \vec{g}_I}^2 \\
&= \; \sum_{I \in \mathcal{J}} \norm{\vec{f}_I - \vec{g}_I}^2 + 2 \sum_{I \in \mathcal{J}}  \langle \vec{\epsilon}_I, \vec{f}_I - \vec{g}_I \rangle + \sum_{I \in \mathcal{J}}  \norm{\vec{\epsilon}_I}^2 \\
&\leq \; O \left( k \sigma^2 \left( r + \log \frac{n}{\delta} \right) \right) + O\left( \sigma  \log \left( \frac{n}{\delta} \right) \sqrt{k n} \right) +  \sum_{I \in \mathcal{J}} \norm{\vec{\epsilon}_I}^2 \; ,
\end{align*}
so all together now we have
\[
\sum_{I \in \mathcal{K}}  \norm{\vec{g}_I - \vec{y}_I}^2 \leq O \left( k \sigma^2 \left( r + \log \frac{n}{\delta} \right) \right) + O\left( \sigma  \log \left( \frac{n}{\delta} \right) \sqrt{k n} \right) + \norm{\vec{\epsilon}}^2 \; .
\]
Moreover, by the same kinds of calculations, we can expand out the LHS of~\eqref{eq:postDP-opt}:
\begin{align*}
\norm{\vec{\fhat} - \vec{y}}^2 \; &\geq \; \norm{\vec{\fhat} - \vec{f}}^2 + \langle \vec{\epsilon},  \vec{\fhat} - \vec{f} \rangle + \norm{\vec{\epsilon}}^2 \\
& \; \geq \| \vec{\fhat} - \vec{f} \|^2 - O \left( \sqrt{k} \cdot \sigma \sqrt{r + \log \frac{n}{\delta} } \right) \| \vec{\fhat} - \vec{f} \| + \| \vec{\epsilon} \|^2
\end{align*}
and hence, combining, cancelling, and moving terms around, we get that
\begin{align*}
\| \vec{\fhat}  - \vec{f} \|^2 &\leq O \left( \sqrt{k} \cdot \sigma \sqrt{r + \log \frac{n}{\delta} } \right) \left( 1 + \| \vec{\fhat} - \vec{f}\| \right) + O\left( \sigma  \log \left( \frac{n}{\delta} \right) \sqrt{k n} \right) \; .
\end{align*}
By Lemma \ref{lem:quadratic}, this implies that
\[ \| \vec{\fhat} - \vec{f} \|^2 \leq O \left( k \sigma^2 \left( r + \log \frac{n}{\delta} \right) + \sigma  \log \left( \frac{n}{\delta} \right) \sqrt{k n} \right) \; ,\]
with probability $1 - \delta$, as claimed.
\end{proof}

We remark that \textsc{Postprocessing} can also be run on the output of \textsc{GreedyMerging} to decrease the number of pieces from $O(k)$ to $2k + 1$ if so desired.
The proof that it maintains similar statistical guarantees is almost identical to the one presented above.

\section{Obtaining agnostic guarantees}
\label{sec:agnostic}
In this section, we demonstrate how to show agnostic guarantees for the algorithms in the previous sections.
Recall now $f$ is arbitrary, and $f^\ast$ is a $k$-piecewise linear function which obtains the best approximation in mean-squared error to $f$ amongst all $k$-piecewise linear functions.
For all $i = 1,\ldots, n$, define $\zeta_i = f(\vec{x}_i) - f^\ast (\vec{x}_i)$ to be the error at data point $i$ of the approximation, so that for all $i$, we have 
\begin{equation}
\label{eq:agnostic-model}
y_i = f^\ast (\vec{x}_i) + \zeta_i + \epsilon_i \; .
\end{equation}
By definition, we have that $\| \vec{\zeta} \|^2 = n \cdot \OPT_k$.

As a warm-up, we first show the following agnostic guarantee for the exact DP:
\begin{theorem}
\label{thm:DP-agnostic}
Fix $\delta > 0$.
Let $\fLS$ be the $k$-piecewise linear function returned by the exact DP.
Then, with probability $1 - \delta$, we have 
\[
\MSE (\fLS) \leq O \left( \sigma^2 \frac{kr + k \log \frac{n}{\delta}}{n} + \sigma \log \left( \frac{1}{\delta} \right) \sqrt{\frac{\OPT_k}{n}} + \OPT_k \right)
\]
\end{theorem}
In the regimes that are most interesting, i.e., when $\OPT_k$ is small, the middle term does not contribute significantly to the error.
\begin{proof}
The overall proof structure stays the same.
From the definition of the least-squares fit, we have
\begin{align*}
\| \vec{y} - \vecfLS \|^2 &\leq \| \vec{y} - \vec{f}^\ast \|^2 \\
&= \| \vec{\epsilon} + \vec{\zeta} \|^2 \\
&= \| \vec{\epsilon} \|^2 + 2 \langle \vec{\epsilon}, \vec{\zeta} \rangle^2 + n \cdot \OPT_k \\
&\leq \| \vec{\epsilon} \|^2 + O \left( \sigma \log \frac{1}{\delta} \right) \left( n \cdot \OPT_k \right)^{1/2} + n \cdot \OPT_k \; ,
\end{align*}
with probability $1 - O(\delta)$, where the last inequality follows from the $r = 1$ case of Lemma \ref{lem:incoherence}.
The LHS can be expanded out exactly as before in~\eqref{eq:DP-LHS}, and putting these two sides together we obtain that
\begin{align*}
\| \vecfLS - \vec{f} \|^2 \; &\leq \; 2 \langle \vec{\epsilon}, \vecfLS - \vec{f} \rangle + O \left(\sigma \log \frac{1}{\delta} \right) \left( n \cdot \OPT_k \right)^{1/2} + n \cdot \OPT_k\\
&\leq O \left( \sigma \sqrt{kr + k \log \frac{n}{\delta}} \right) \cdot \| \vecfLS - \vec{f} \| + O \left(  \sigma \log \frac{1}{\delta} \right) \left( n \cdot \OPT_k \right)^{1/2} + n \cdot \OPT_k
\end{align*}
with probability at least $1 - O(\delta)$ and so by Lemma \ref{lem:quadratic} we obtain that
\[
\| \vecfLS - \vec{f} \|^2 \leq O \left( \sigma^2 \left( kr + k \log \frac{n}{\delta}\right) + \sigma \log \left( \frac{1}{\delta} \right) \left( n \cdot \OPT_k \right)^{1/2} + n \cdot \OPT_k \right) \;.
\]
Dividing both sides by $n$ yields the desired conclusion.
\end{proof}
Reviewing the proof, we observe the following general pattern: in all upper bounds we wish to obtain (generally for formulas which occur on the RHS of the equations above) we obtain new $\OPT_k$ terms, whereas in the lower bounds (generally for formulas on the LHS of the equations above) nothing changes.
In fact, all folllowing proofs exhibit this pattern.
We first establish some useful concentration bounds.
By the same union-bound technique used throughout this paper, one can show the following.
We omit the proof for conciseness.
\begin{lemma}
\label{lem:agnostic-incoherence}
Fix $\delta > 0.$
Let $\epsilon_i$ and $\zeta_i$ be as in~\eqref{eq:agnostic-model}, for $i = 1, \ldots, n$.
Then, with probability $1 - \delta$, we have that for all collections of $k$ disjoint intervals $J_1, \ldots, J_k$, the following holds:
\begin{align*}
\sum_{\ell = 1}^k \langle \vec{\epsilon}_{J_\ell}, \vec{\zeta}_{J_\ell} \rangle &\leq O \left( k \sigma \log \frac{n}{\delta} \right) \cdot \left( \sum_{\ell = 1}^k \| \vec{\zeta}_{J_\ell} \|^2 \right)^{1/2} \\
&\leq O \left( k \sigma \log \frac{n}{\delta} \right) \cdot \left( n \cdot \OPT_k \right)^{1/2}
\end{align*}
\end{lemma}

Observe that with this lemma, we can easily adapt the proof of Theorem \ref{thm:DP-agnostic} to show the following agnostic version of Lemma \ref{lem:flat}:
\begin{lemma}
\label{lem:agnostic-flat}
Fix $\delta > 0$. Then with probability $1 - \delta$, we have that for all disjoint sets of $k$ intervals $J_1, \ldots, J_k$ of $[n]$ so that $f^\ast$ is flat on each $J_\ell$, the following inequality holds:
\[
\sum_{\ell = 1}^k \| \vecfLS_{J_\ell}  - \vec{f}_{J_\ell} \|^2 \leq O \left( \sigma^2 \left( kr + k \log \frac{n}{\delta}\right) + k \sigma \log \left( \frac{n}{\delta} \right) \left( n \cdot \OPT_k \right)^{1/2} + n \cdot \OPT_k \right) \; .
\]
\end{lemma}

With this lemma, we can now prove the following theorem for our greedy algorithm, which shows that in the presence of model misspecification, our algorithm still performs at a good rate:
\begin{theorem}
Let $\delta > 0$, and let $\fhat$ be the estimator returned by \textsc{GreedyMerging}.
Let $m = (2 + \frac{2}{\tau}) k + \gamma$ be the number of pieces in $\fhat$.
Then, with probability $1 - \delta$, we have that
\[\MSE (\fhat) \leq O\left( \sigma^2 \left( \frac{m (r + \log (n / \delta)) }{n} \right) + \sigma \frac{\tau + \sqrt{k}}{\sqrt{n}} \log \left( \frac{n}{\delta} \right) + k \sigma \log \left( \frac{n}{\delta} \right) \sqrt{\frac{\OPT_k }{n}} + \tau \cdot \OPT_k \right) \; .\]
\end{theorem}
\begin{proof}
As before, we condition on the event that Corollaries \ref{cor:unif-err-bound}, \ref{cor:int-incoherence}, and \ref{cor:sim-incoherence}, and Lemmas \ref{lem:agnostic-incoherence} and \ref{lem:agnostic-flat} all hold with error parameter $O(\delta)$, so that together they all hold with probability at least $1 - \delta$.
We also let $\setI, \mathcal{F}, \mathcal{J}_1,$ and $\mathcal{J}_2$ denote the same quantities as before, except we define the partition with respect to the jumps of $f^\ast$ instead of $f$, since the latter does not have well-defined jumps.
For instance, $\mathcal{F}$ is the set of intervals in $\setI$ on which $f^\ast$ has no jumps.

We again bound the error on these three sets separately.
First, by Lemma \ref{lem:agnostic-flat} we have that
\begin{equation}
\label{eq:-agnostic-flat}
\sum_{I \in \mathcal{F}} \| \vec{f}_I - \vec{\fhat}_I \|^2 \; \leq \; O \left( \sigma^2 \left( mr + m \log \frac{n}{\delta}\right) + m \sigma \log \left( \frac{n}{\delta} \right) \left( n \cdot \OPT_k \right)^{1/2} + n \cdot \OPT_k \right) \; .
\end{equation}

Second, we bound the error on $\mathcal{J}_1$.
By modifying the calculations as those preceding~\eqref{eq:jump1} in the same way as we did above in the proof of Theorem \ref{thm:DP-agnostic}, we may show that
\begin{equation}
\label{eq:agnostic-jump1}
\sum_{I \in \mathcal{J}_1}  \| \vec{f}_I - \vec{\fhat}_I \|^2 \leq \sigma^2 \left( k + O \left( \log \frac{n}{\delta} \right) \sqrt{m} \right) + O \left(k \sigma \log \frac{n}{\delta} \left( n \cdot \OPT_k \right)^{1/2} \right) + n \cdot \OPT_k \; .
\end{equation}
Finally, we bound the error on $\mathcal{J}_2$.
Fix an $I \in \mathcal{J}_2$, and let $M_1, \ldots, M_{k / \tau}$ be as before.
Recall that this proof had two components: an upper bound for the $M_1, \ldots, M_\ell$ and a lower bound for $I$.
We first compute the upper bound.
By following the calculations for~\eqref{eq:upper-bound} and using Lemma~\ref{lem:agnostic-flat}, we have
\begin{align*}
\sum_{\ell = 1}^{k / \tau} \left( \| \vec{y}_{M_\ell} - \vecfLS_{M_\ell} \|^2 - s^2 |M_\ell| \right) \leq &\;O \left( \frac{k}{\tau} \sigma^2 \left(r+ \log \left( \frac{n}{\delta} \right) \right) + \sigma \log\left( \frac{n}{\delta} \right)\sqrt{n} \right. \\
&\left. + \frac{k}{\tau} \sigma \log \left( \frac{n}{\delta} \right) \left( n \cdot \OPT_k \right)^{1/2} + n \cdot \OPT_k \right)
\end{align*}
The lower bound is in fact unchanged: we may use~\eqref{eq:lower-bound} as stated.
Putting these two terms together and simplifying as we did previously, we obtain that
\begin{align*}
\| \vec{f}_I - \vec{\fhat}_I \|^2 \leq\; & O \left( \sigma^2 \left(r + \log \left( \frac{n}{\delta} \right) \right) + \sigma \log \frac{n}{\delta} \left( \frac{\tau \sqrt{n}}{k}  + \sqrt{|I|} \right) \right.\\
&\left.+ \sigma \log \left( \frac{n}{\delta} \right) \left( n \cdot \OPT_k \right)^{1/2} + \frac{\tau}{k} n \cdot \OPT_k \right) \; .
\end{align*}
As before, summing up all the bounds we have achieved proves the desired claim.
\end{proof}

Through virtually identical methods we also obtain the same upper bound for \textsc{BucketGreedyMerge} and \textsc{PostprocessedBucketGreedyMerge}.
We state the results but omit their proofs for this reason.

\begin{theorem}
Let $\fhat$ be the $m$-piecewise linear function that is returned by \textsc{BucketGreedyMerge}, where $m = (2 + \frac{2}{\tau}) k \log n + \gamma$.
Then, with probability $1 - \delta$, we have
\[\MSE (\fhat) \leq O\left( \sigma^2 \left( \frac{m( r +  \log (n / \delta) ) }{n} \right) + \sigma \sqrt{\frac{k}{n}} \log \left( \frac{n}{\delta} \right) + k \sigma \log \left( \frac{n}{\delta} \right) \sqrt{\frac{\OPT_k }{n}} + \tau \cdot \OPT_k \right) \; .\]
\end{theorem}

\begin{theorem}
Fix $\delta > 0$.
Let $\fhat_p$ be the estimator output by \textsc{PostprocessedBucketGreedyMerge}.
Then, with probability $1 - \delta$, we have
\[
\MSE (\fhat_p) \leq O \left( \sigma^2 \frac{k \left( r + \log \frac{n}{\delta} \right)}{n} + \sigma  \log \left( \frac{n}{\delta} \right) \sqrt{\frac{k}{n}} + k \sigma \log \left( \frac{n}{\delta} \right) \sqrt{\frac{\OPT_k }{n}} + \tau \cdot \OPT_k\right)
 \; .\]
\end{theorem}


\pgfplotsset{ignore legend/.style={every axis legend/.code={\renewcommand\addlegendentry[2][]{}}}}

\pgfplotsset{resultplot/.style={%
  scale only axis,
  tick label style={font=\scriptsize},
  yticklabel shift={-2pt},
  enlarge x limits=false,
  xlabel shift=-3pt,
  ylabel shift=-5pt,
  label style={font=\scriptsize},
  title style={font=\scriptsize,yshift=-4pt},
  grid=major,
  grid style={dotted,gray,thin},
  legend cell align=left,
  legend transposed=true,
  legend columns=-1,
  legend style={anchor=south west,at={(-1.8,-.5)}},
  every axis plot/.append style={line width=1pt,mark size=2pt},
  cycle list={{red,mark=square},{green,mark=asterisk},{olive,mark=+},{blue,mark=o}}}}
\pgfplotsset{resultplot1/.style={%
  resultplot,
  height=\plotheight,
  width=\plotwidth}}

\newlength{\plotwidth}
\setlength{\plotwidth}{3cm}
\newlength{\plotheight}
\setlength{\plotheight}{2.6cm}
\newlength{\plotxspacing}
\setlength{\plotxspacing}{1.3cm}
\newlength{\plotyspacing}
\setlength{\plotyspacing}{1.3cm}

\begin{figure*}[!t]
\centering

\pgfplotsset{mseplot/.style={%
  resultplot1,
  xlabel=$n$,
  ylabel=MSE}}
\pgfplotsset{mseratioplot/.style={%
  resultplot1,
  xlabel=$n$,
  ylabel=Relative MSE ratio}}
\pgfplotsset{timeplot/.style={%
  resultplot1,
  xlabel=$n$,
  ylabel=Running time (s)}}
\pgfplotsset{timeratioplot/.style={%
  resultplot1,
  xlabel=$n$,
  ylabel=Speed-up}}

\begin{tikzpicture}

\begin{loglogaxis}[mseplot,name=mseplot1,ignore legend]
\iftoggle{plotexperiments}{
\addplot table[x=n, y=mse_mean] {plot_data/experiments2_k10_merging_k.txt};
\addplot table[x=n, y=mse_mean] {plot_data/experiments2_k10_merging_2k.txt};
\addplot table[x=n, y=mse_mean] {plot_data/experiments2_k10_merging_4k.txt};
\addplot table[x=n, y=mse_mean] {plot_data/experiments2_k10_exactdp.txt};
}{}
\end{loglogaxis}

\begin{semilogxaxis}[mseratioplot,ymax=10,name=mseratioplot1,at=(mseplot1.north east),anchor=north west,xshift=\plotxspacing,ignore legend]
\iftoggle{plotexperiments}{
\addplot table[x=n, y=mse_ratio] {plot_data/experiments2_k10_merging_k.txt};
\addplot table[x=n, y=mse_ratio] {plot_data/experiments2_k10_merging_2k.txt};
\addplot table[x=n, y=mse_ratio] {plot_data/experiments2_k10_merging_4k.txt};
}{}
\end{semilogxaxis}

\begin{loglogaxis}[timeplot,name=timeplot1,at=(mseratioplot1.north east),anchor=north west,xshift=\plotxspacing,ignore legend]
\iftoggle{plotexperiments}{
\addplot table[x=n, y=time_mean] {plot_data/experiments2_k10_merging_k.txt};
\addplot table[x=n, y=time_mean] {plot_data/experiments2_k10_merging_2k.txt};
\addplot table[x=n, y=time_mean] {plot_data/experiments2_k10_merging_4k.txt};
\addplot table[x=n, y=time_mean] {plot_data/experiments2_k10_exactdp.txt};
}{}
\end{loglogaxis}

\begin{loglogaxis}[timeratioplot,name=timeratioplot1,at=(timeplot1.north east),anchor=north west,xshift=\plotxspacing,ignore legend]
\iftoggle{plotexperiments}{
\addplot table[x=n, y=time_ratio] {plot_data/experiments2_k10_merging_k.txt};
\addplot table[x=n, y=time_ratio] {plot_data/experiments2_k10_merging_2k.txt};
\addplot table[x=n, y=time_ratio] {plot_data/experiments2_k10_merging_4k.txt};
}{}
\end{loglogaxis}

\begin{loglogaxis}[mseplot,name=mseplot2,at=(mseplot1.south west),anchor=north west,yshift=-\plotyspacing,ignore legend]
\iftoggle{plotexperiments}{
\addplot table[x=n, y=mse_mean] {plot_data/experiments5_k5_merging_k.txt};
\addplot table[x=n, y=mse_mean] {plot_data/experiments5_k5_merging_2k.txt};
\addplot table[x=n, y=mse_mean] {plot_data/experiments5_k5_merging_4k.txt};
\addplot table[x=n, y=mse_mean] {plot_data/experiments5_k5_exactdp.txt};
}{}
\end{loglogaxis}

\begin{semilogxaxis}[mseratioplot,ymax=10,name=mseratioplot2,at=(mseratioplot1.south west),anchor=north west,yshift=-\plotyspacing,ignore legend]
\iftoggle{plotexperiments}{
\addplot table[x=n, y=mse_ratio] {plot_data/experiments5_k5_merging_k.txt};
\addplot table[x=n, y=mse_ratio] {plot_data/experiments5_k5_merging_2k.txt};
\addplot table[x=n, y=mse_ratio] {plot_data/experiments5_k5_merging_4k.txt};
}{}
\end{semilogxaxis}

\begin{loglogaxis}[timeplot,ytick={0.0001,0.01,1,100},name=timeplot2,at=(timeplot1.south east),anchor=north east,yshift=-\plotyspacing]
\iftoggle{plotexperiments}{
\addplot table[x=n, y=time_mean] {plot_data/experiments5_k5_merging_k.txt};
\addplot table[x=n, y=time_mean] {plot_data/experiments5_k5_merging_2k.txt};
\addplot table[x=n, y=time_mean] {plot_data/experiments5_k5_merging_4k.txt};
\addplot table[x=n, y=time_mean] {plot_data/experiments5_k5_exactdp.txt};
\legend{Merging $k$, Merging $2k$, Merging $4k$, Exact DP}
}{}
\end{loglogaxis}

\begin{loglogaxis}[timeratioplot,name=timeratioplot2,at=(timeratioplot1.south east),anchor=north east,yshift=-\plotyspacing,ignore legend]
\iftoggle{plotexperiments}{
\addplot table[x=n, y=time_ratio] {plot_data/experiments5_k5_merging_k.txt};
\addplot table[x=n, y=time_ratio] {plot_data/experiments5_k5_merging_2k.txt};
\addplot table[x=n, y=time_ratio] {plot_data/experiments5_k5_merging_4k.txt};
}{}
\end{loglogaxis}

\node [at=(mseratioplot1.north east),anchor=south west,xshift=-1cm,yshift=.2cm] {Piecewise constant};
\node [at=(mseratioplot2.north east),anchor=south west,xshift=-1cm,yshift=.2cm] {Piecewise linear};
\end{tikzpicture}
\iftoggle{fullversion}{
}{
\vspace{-.5cm}}
\caption{Experiments with synthetic data: results for piecewise constant models with $k=10$ segments (top row) and piecewise linear models with $k=5$ segments (bottom row, dimension $d=10$).
Compared to the exact dynamic program, the MSE achieved by the merging algorithm is worse but stays within a factor of 2 to 4 for a sufficient number of output segments.
The merging algorithm is significantly faster and achieves a speed-up of about $10^3\times$ compared to the dynamic program for $n=10^4$.
This leads to a significantly better trade-off between statistical and computational performance (see also Figure \ref{fig:plots2}).}
\label{fig:plots1}
\end{figure*}
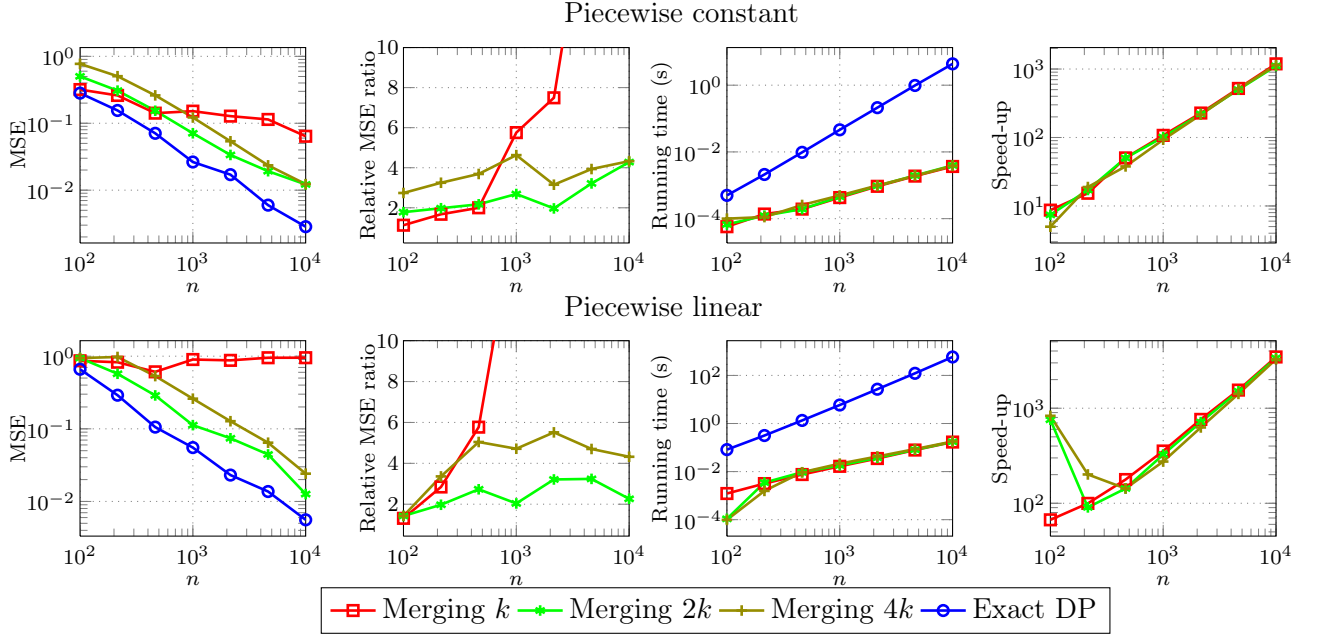

\iftoggle{fullversion}{
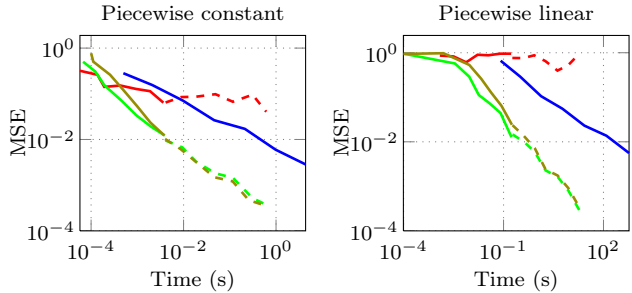
\begin{figure}[htb]
}{
\begin{figure}
}
\centering

\pgfplotsset{tradeoffplot/.style={%
  resultplot1,
  ymin=0.0001,
  ytick={0.0001,0.01,1.0},
  cycle list={{red,mark=none},{green,mark=none},{olive,mark=none},{blue,mark=none},{red,mark=none,dashed},{green,mark=none,dashed},{olive,mark=none,dashed}},
  xlabel=Time (s),
  ylabel=MSE}}

\begin{tikzpicture}

\begin{loglogaxis}[title=Piecewise constant,tradeoffplot,name=tradeoffplot1,ignore legend]
\iftoggle{plotexperiments}{
\addplot table[x=time_mean, y=mse_mean] {plot_data/experiments2_k10_merging_k.txt};
\addplot table[x=time_mean, y=mse_mean] {plot_data/experiments2_k10_merging_2k.txt};
\addplot table[x=time_mean, y=mse_mean] {plot_data/experiments2_k10_merging_4k.txt};
\addplot table[x=time_mean, y=mse_mean] {plot_data/experiments2_k10_exactdp.txt};
\addplot table[x=time_mean, y=mse_mean] {plot_data/experiments2_k10_extra_merging_k.txt};
\addplot table[x=time_mean, y=mse_mean] {plot_data/experiments2_k10_extra_merging_2k.txt};
\addplot table[x=time_mean, y=mse_mean] {plot_data/experiments2_k10_extra_merging_4k.txt};
}{}
\end{loglogaxis}

\begin{loglogaxis}[title=Piecewise linear,tradeoffplot,name=tradeoffplot2,at=(tradeoffplot1.south east),anchor=south west,xshift=1.3cm,ignore legend]
\iftoggle{plotexperiments}{
\addplot table[x=time_mean, y=mse_mean] {plot_data/experiments5_k5_merging_k.txt};
\addplot table[x=time_mean, y=mse_mean] {plot_data/experiments5_k5_merging_2k.txt};
\addplot table[x=time_mean, y=mse_mean] {plot_data/experiments5_k5_merging_4k.txt};
\addplot table[x=time_mean, y=mse_mean] {plot_data/experiments5_k5_exactdp.txt};
\addplot table[x=time_mean, y=mse_mean] {plot_data/experiments5_k5_extra_merging_k.txt};
\addplot table[x=time_mean, y=mse_mean] {plot_data/experiments5_k5_extra_merging_2k.txt};
\addplot table[x=time_mean, y=mse_mean] {plot_data/experiments5_k5_extra_merging_4k.txt};
}{}
\end{loglogaxis}
\end{tikzpicture}
\iftoggle{fullversion}{
}{
\vspace{-.85cm}}
\caption{Computational vs.\ statistical efficiency in the synthetic data experiments.
The solid lines correspond to the data in Figure~\ref{fig:plots1}, the dashed lines show the results from additional runs of the merging algorithms for larger values of $n$.
The merging algorithm achieves the same MSE as the dynamic program about $100 \times$ faster if a sufficient number of samples is available.}
\label{fig:plots2}
\end{figure}

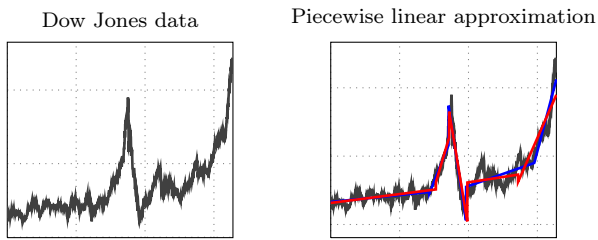
\begin{figure}
\centering

\pgfplotsset{djplot/.style={%
  resultplot1,
  ticks=none,
  cycle list={{darkgray,mark=none},{blue,mark=none},{red,mark=none}}}}

\begin{tikzpicture}

\begin{axis}[djplot,title=Dow Jones data,name=djplot1]
\iftoggle{plotexperiments}{
\addplot table [x expr=\coordindex, y index=0] {plot_data/dj_data.txt};
}{}
\end{axis}

\begin{axis}[djplot,title=Piecewise linear approximation,at=(djplot1.south east),anchor=south west,xshift=1.3cm]
\iftoggle{plotexperiments}{
\addplot table [x expr=\coordindex, y index=0] {plot_data/dj_data.txt};
\addplot table [x expr=\coordindex, y index=0] {plot_data/dj_dp.txt};
\addplot table [x expr=\coordindex, y index=0] {plot_data/dj_merging_k.txt};
}{}
\end{axis}
\end{tikzpicture}
\vspace{-.2cm}
\caption{Results of fitting a 5-piecewise linear function ($d=2$) to a Dow Jones time series.
The merging algorithm produces a fit that is comparable to the dynamic program and is about $200\times$ faster ($0.013$ vs.\ $3.2$ seconds).}
\label{fig:plots3}
\end{figure}

\section{Experiments}
\label{sec:experiments}
In addition to our theoretical analysis above, we also study the empirical performance of our new estimator for segmented regression on both real and synthetic data.
As baseline, we compare our estimator (\textsc{GreedyMerging}) to the dynamic programming approach.
Since our algorithm combines both combinatorial and linear-algebraic operations, we use the Julia programming language\footnote{\url{http://julialang.org/}} (version 0.4.2) for our experiments because Julia achieves performance close to C on both types of operations.
All experiments were conducted on a laptop computer with a 2.8 GHz Intel Core i7 CPU and 16 GB of RAM.

\paragraph{Synthetic data.}
Experiments with synthetic data allow us to study the statistical and computational performance of our estimator as a function of the problem size $n$.
Our theoretical bounds indicate that the worst-case performance of the merging algorithm scales as $O(\frac{k d}{n} + \sqrt{k / n} \log n)$ for constant error variance.
Compared to the $O(\frac{kd}{n})$ rate of the dynamic program, this indicates that the relative performance of our algorithm can depend on the number of features $d$.
Hence we use two types of synthetic data: a piecewise-constant function $f$ (effectively $d=1$) and a piecewise linear function $f$ with $d=10$ features.

We generate the piecewise constant function $f$ by randomly choosing $k = 10$ integers from the set $\{1, \ldots, 10\}$ as function value in each of the $k$ segments.\footnote{We also repeated the experiment for other values of $k$.
Since the results are not qualitatively different, we only report the $k=10$ case here.}
Then we draw $n/k$ samples from each segment by adding an i.i.d.\ Gaussian noise term with variance 1 to each sample.

For the piecewise linear case, we generate a $n \times d$ data matrix $\vec{X}$ with i.i.d.\ Gaussian entries ($d = 10$).
In each segment $I$, we choose the parameter values $\vec{\beta}_I$ independently and uniformly at random from the interval $[-1, 1]$.
So the true function values in this segment are given by $\vec{f}_I = \vec{X}_I \vec{\beta}_I$.
As before, we then add an i.i.d.\ Gaussian noise term with variance 1 to each function value.

Figure \ref{fig:plots1} shows the results of the merging algorithm and the exact dynamic program for sample size $n$ ranging from $10^2$ to $10^4$.
Since the merging algorithm can produce a variable number of output segments, we run the merging algorithm with three different parameter settings corresponding to $k$, $2k$, and $4k$ output segments, respectively.
As predicted by our theory, the plots show that the exact dynamic program has a better statistical performance.
However, the MSE of the merging algorithm with $2k$ pieces is only worse by a factor of $2$ to $4$, and this ratio empirically increases only slowly with $n$ (if at all).
The experiments also show that forcing the merging algorithm to return at most $k$ pieces can lead to a significantly worse MSE.

In terms of computational performance, the merging algorithm has a significantly faster running time, with speed-ups of more than $1,000\times$ for $n=10^4$ samples.
As can be seen in Figure \ref{fig:plots2}, this combination of statistical and computational performance leads to a significantly improved trade-off between the two quantities.
When we have a sufficient number of samples, the merging algorithm achieves a given MSE roughly $100\times$ faster than the dynamic program.

\paragraph{Real data.}
We also investigate whether the merging algorithm can empirically be used to find linear trends in a real dataset.
We use a time series of the Dow Jones index as input, and fit a piecewise linear function ($d=2$) with 5 segments using both the dynamic program and our merging algorithm with $k=5$ output pieces.
As can be seen from Figure \ref{fig:plots3}, the dynamic program produces a slightly better fit for the rightmost part of the curve, but the merging algorithm identifies roughly the same five main segments.
As before, the merging algorithm is significantly faster and achieves a $200\times$ speed-up compared to the dynamic program (0.013 vs 3.2 seconds).

\section*{Acknowledgements}
Part of this research was conducted while Ilias Diakonikolas was at the University of Edinburgh, Jerry Li was an intern at Microsoft Research Cambridge (UK), and Ludwig Schmidt was visiting the EECS department at UC Berkeley.

Jayadev Acharya was supported by a grant from the MIT-Shell Energy Initiative.
Ilias Diakonikolas was supported in part by EPSRC grant EP/L021749/1, a Marie Curie Career Integration Grant, and a SICSA grant.
Jerry Li was supported by NSF grant CCF-1217921, DOE grant DE-SC0008923, NSF CAREER Award CCF-145326, and a NSF Graduate Research Fellowship.
Ludwig Schmidt was supported by grants from the MIT-Shell Energy Initiative, MADALGO, and the Simons Foundation.

\bibliographystyle{alpha}

\bibliography{main}



\end{document}